%% file: iclr2027_conference.tex
\documentclass{article} 
\usepackage{iclr2027_conference,times}

\let\savedaddcontentsline\addcontentsline
\let\addcontentsline\savedaddcontentsline

\usepackage{titletoc}

\dottedcontents{appsection}
  [18pt]{\addvspace{4pt}\bfseries}{18pt}{5pt}

\dottedcontents{appsubsection}
  [42pt]{\addvspace{4pt}\bfseries}{24pt}{5pt}

\dottedcontents{appsubsubsection}
  [72pt]{\addvspace{4pt}}{30pt}{5pt}

\newcommand{\appendixtableofcontents}{
  \startcontents[appendix]
  \section*{Appendix Contents}
  \hrule height 0.4pt
  \vspace{6pt}

  \printcontents[appendix]{app}{1}[3]{
    \setlength{\parskip}{0pt}
    \linespread{1.25}\selectfont
  }

  \vspace{8pt}
  \hrule height 0.4pt
}

\usepackage{amsmath}
\usepackage{amssymb}
\usepackage{amsthm}
\usepackage{hyperref}
\usepackage{wrapfig}

\usepackage{graphicx}
\usepackage{pifont}

\providecommand{\xmark}{\ding{55}}
\usepackage{nicematrix}
\usepackage{tikz}
\usepackage[table]{xcolor}
\definecolor{implicitcolor}{HTML}{b11f23}
\definecolor{explicitcolor}{HTML}{5178a1}
\definecolor{standardcolor}{HTML}{148e6f}
\usepackage{array}
\usepackage{tabularx}
\usepackage{booktabs}
\usepackage{multirow}
\theoremstyle{plain}
\newtheorem{theorem}{Theorem}[section]

\newtheorem{lemma}[theorem]{Lemma}
\newtheorem{corollary}[theorem]{Corollary}

\theoremstyle{definition}

\newtheorem{assumption}[theorem]{Assumption}
\usepackage{enumitem}
\theoremstyle{remark}

\input{math_commands.tex}

\usepackage{url}
\usepackage{tcolorbox}

\title{Opinion Leader Dynamics: How Sparse Attention Shapes Token Clustering}

\author{
Jingkun Liu$^{1,2}$, Yue Song$^{1}$\thanks{Denotes corresponding author.} \\
$^{1}$Tsinghua University, $^{2}$Shanghai Jiao Tong University
}

\iclrfinalcopy 
\begin{document}

\maketitle
\fancyhead{}
\renewcommand{\headrulewidth}{0pt}

\begin{abstract}

Sparse attention reduces the quadratic cost of global self-attention while retaining strong empirical performance, but how its restricted interactions shape the evolution of token representations remains theoretically underexplored. Modeling tokens as particles on the unit sphere, we introduce \textbf{opinion leader dynamics}, a framework that identifies two mechanisms through which token groups converge internally while maintaining distinct limiting directions. In the explicit model, fixed representatives induce a potential that attracts tokens toward distinct local maxima. In the implicit model, disconnected interaction groups evolve toward separate consensus directions. We formulate both models as \emph{reverse Wasserstein gradient flows} and establish exponential convergence under suitable conditions. We further connect these theoretical predictions to token evolution in frontier sparse-attention LLMs that motivate our framework. Across four benchmarks, \textcolor{explicitcolor}{\textbf{Kimi-K3}},  \textcolor{implicitcolor}{\textbf{MiniMax-M3}}, and \textcolor{explicitcolor}{\textbf{DeepSeek-}}\textcolor{implicitcolor}{\textbf{V4-Flash}} consistently exhibit clearer cluster separation and higher clustering scores than the dense-attention model \textcolor{standardcolor}{\textbf{GLM-4.7-Flash}} in projected token representations. These observations support the relevance of the predicted multiple-group structure to trained frontier LLMs, while finite-particle simulations illustrate the theoretical convergence behavior. Together, our results connect restricted token interactions to distinct group-level attractors, providing a dynamical account of how sparse attention can support alignment within groups while preserving separation between them. Code is available at \url{https://github.com/Jingkun-Liu/Opinion-Leader-Dynamics.git}

\end{abstract}


\input{1_introduction.tex}

\input{2_opinion_leader_dynamics.tex}
\input{3_experiments.tex}

\input{4_conclusions.tex}

\bibliography{iclr2027_conference}
\bibliographystyle{iclr2027_conference}

\newpage

\appendix
\input{Appendix.tex}

\end{document}

%% file: math_commands.tex
\usepackage{amsmath,amsfonts,bm}

\def\eqref#1{equation~\ref{#1}}

\def\1{\bm{1}}

\DeclareMathAlphabet{\mathsfit}{\encodingdefault}{\sfdefault}{m}{sl}
\SetMathAlphabet{\mathsfit}{bold}{\encodingdefault}{\sfdefault}{bx}{n}



%% file: 1_introduction.tex
\begin{figure}[htbp]
    \centering
    \includegraphics[width=.9\linewidth]{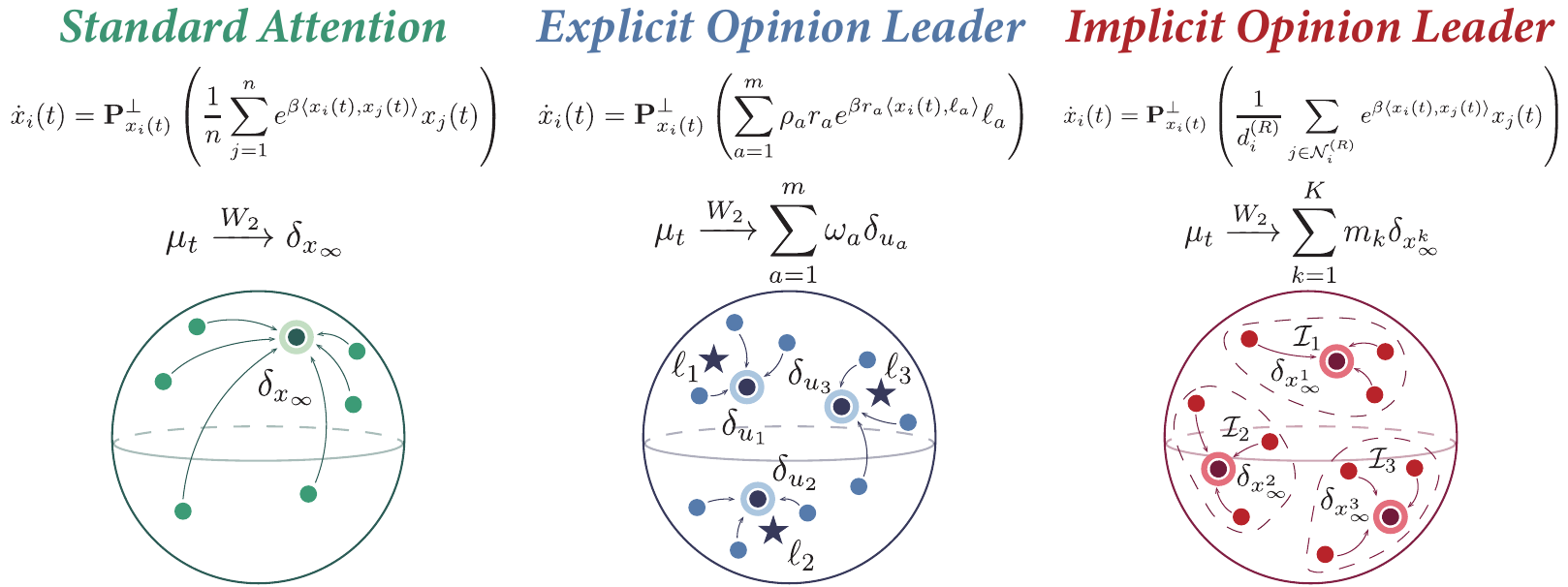}
    \caption{Comparison of three attention dynamics under the respective assumptions of our analysis. \textbf{\emph{Left:}} \textcolor{standardcolor}{\textbf{Standard Attention}} drives tokens toward a single consensus $x_\infty$. \textbf{\emph{Middle:}} \textcolor{explicitcolor}{\textbf{Explicit opinion leader dynamics}} guide token groups toward distinct local maxima $u_a$ of the potential induced by fixed representatives $z_a=r_a\ell_a$. \textbf{\emph{Right:}} \textcolor{implicitcolor}{\textbf{Implicit opinion leader dynamics}} align tokens within each disconnected component $\mathcal I_k$, yielding distinct consensus directions $x_\infty^k$.} 
    \label{fig:teaser}
\end{figure}

\section{Introduction}

Transformers~\citep{vaswani2017attention} underpin modern large language models (LLMs)~\citep{achiam2023gpt,team2026kimi,xu2026deepseek,lai2026minimax,zeng2026glm}, with self-attention governing how tokens exchange and integrate information. Global self-attention allows every token to interact with all other tokens, but its quadratic computational cost limits its efficiency on long sequences. Sparse attention addresses this limitation by restricting interactions or organizing them through compressed representations, often retaining strong empirical performance~\citep{tay2022efficient,tang2024quest,lai2025flexprefill}. These changes affect not only the cost of computation but also the pathways through which tokens influence one another. This raises a fundamental question: \emph{how do restricted interactions shape the clustering and long-term evolution of token representations?}

Recent theoretical work provides a starting point by modeling tokens as interacting particles on the unit sphere and studying their mean-field dynamics~\citep{geshkovski2023emergence,geshkovski2025mathematical,rigollet2026mean}. Within this framework, \citet{chen2025quantitative} establish exponential convergence to a Dirac mass at a single consensus direction under suitable conditions. This result characterizes a regime in which global interactions progressively align token representations. Sparse attention changes the underlying interaction structure, raising the possibility that distinct token groups can retain separate limiting directions. Characterizing the conditions and mechanisms supporting such behavior still remains theoretically underexplored.

\begin{figure}[t]
    \centering
    \includegraphics[width=.99\linewidth]{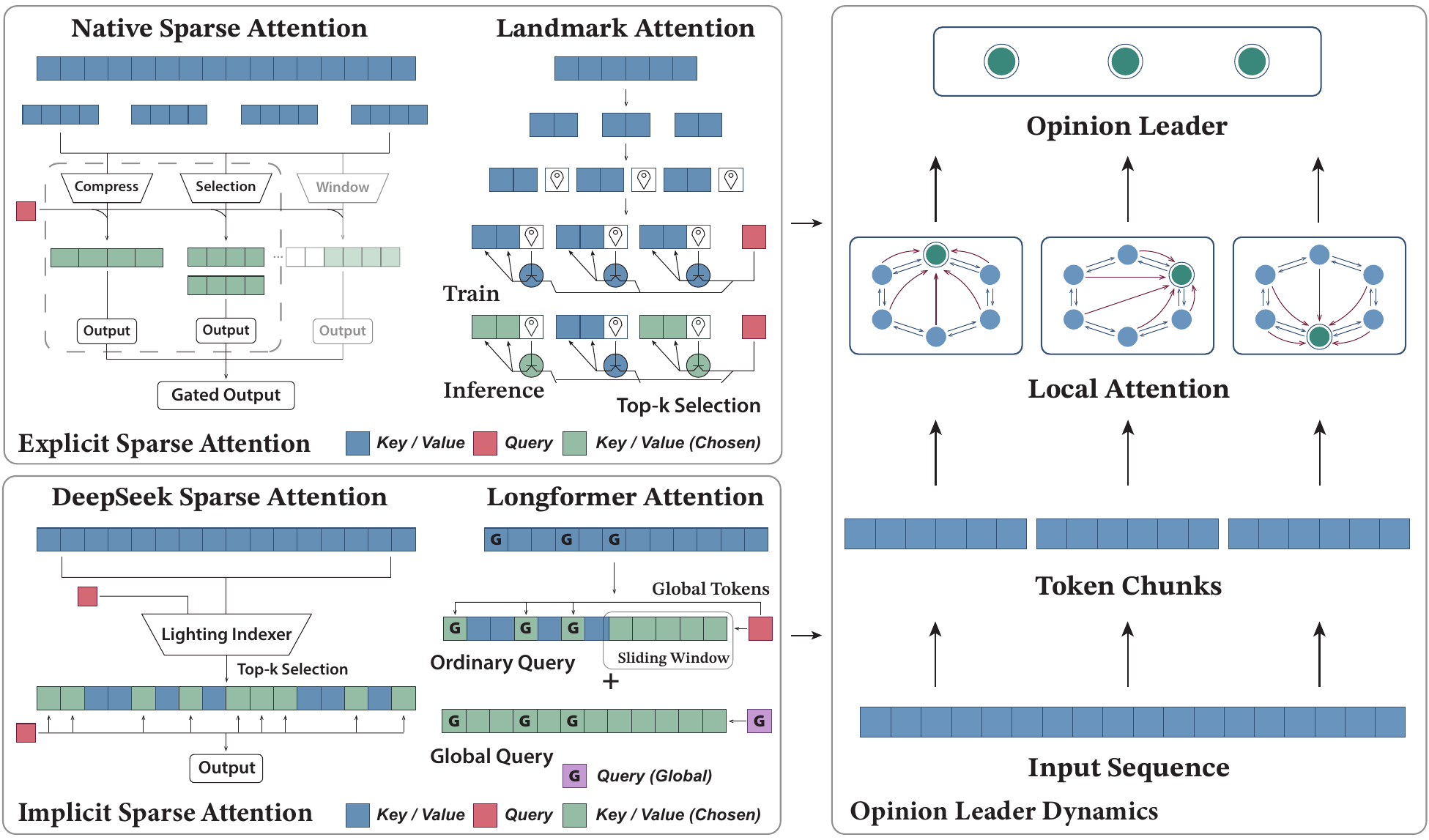}
    \caption{Sparse attention architectures motivating our \emph{opinion leader dynamics} framework. \textbf{\emph{Top left:}} Native Sparse Attention~\citep{yuan2025native} and Landmark Attention~\citep{mohtashami2023landmark} illustrate mechanisms that use compressed or designated representatives, motivating our \emph{explicit} model. \textbf{\emph{Bottom left:}} DeepSeek Sparse Attention~\citep{liu2025deepseek} and Longformer~\citep{beltagy2020longformer} illustrate mechanisms that restrict interactions through selective or structured attention patterns, motivating our \emph{implicit} model. \textbf{\emph{Right:}} A shared schematic of both mechanisms. In the explicit model, fixed representatives induce group-specific opinion leaders; in the implicit model, opinion leaders emerge through groupwise consensus under restricted interactions.
    }
    \label{fig:sparse_architectures}
    \vspace{-3mm}
\end{figure}

Existing sparse attention architectures motivate two complementary mechanisms. Methods based on compressed or landmark representations summarize token groups and use these summaries to aggregate information or guide retrieval~\citep{mohtashami2023landmark,yuan2025native,hu2025hardware}. Such representatives can provide common reference directions toward which tokens evolve. Other methods restrict communication through predefined patterns or content-dependent key selection~\citep{child2019generating,beltagy2020longformer,zaheer2020big,kitaev2020reformer,roy2021efficient,lu2025moba,liu2025deepseek}. These restrictions can favor interaction within token groups, motivating the study of consensus within groups that evolve independently. We isolate these mechanisms through fixed representatives and disconnected interaction components, respectively. These idealizations allow us to characterize two routes to distinct limiting directions and examine their relevance to the token dynamics of frontier sparse-attention LLMs.

\noindent\textbf{Our contributions.} We introduce \textbf{opinion leader dynamics}, comprising explicit dynamics driven by fixed representatives and implicit dynamics driven by restricted mutual interactions. Both model tokens as particles on the unit sphere, with an \emph{opinion leader} denoting a group's limiting direction (see Fig.~\ref{fig:teaser}). Fig.~\ref{fig:sparse_architectures} illustrates the motivating architectures and their shared schematic abstraction.


In the \emph{explicit model}, tokens interact with fixed representatives $\{z_a\}_{a=1}^{m}$, where $z_a=r_a\ell_a$ has magnitude $r_a$ and unit direction $\ell_a$. These representatives jointly induce a time-independent potential, whose attracting local maxima serve as explicit opinion leaders. The token dynamics are given by
\begin{equation}
\dot x_i(t)
=
\mathbf P_{x_i(t)}^{\perp}
\left(
    \sum_{a=1}^{m}
    \rho_a r_a
    e^{\beta r_a\langle x_i(t),\ell_a\rangle}
    \ell_a
\right),
\qquad
x_i(t)\in\mathbb S^{d-1}.
\label{eq:fixed_leader_ode}
\end{equation}
Here, $\rho_a$ is the weight of representative $a$, $\beta>0$ is the inverse temperature, and $\mathbf P_x^\perp(y)=y-\langle x,y\rangle x$ projects onto the tangent space $T_x\mathbb S^{d-1}$. Because all representatives contribute to the potential, the attracting directions need not coincide with the representative directions themselves.

In the \emph{implicit model}, tokens interact through a fixed sparse graph. Fixed interaction coordinates and a radius $R$ define a symmetric graph with $K\geq2$ nontrivial connected components $\{\mathcal I_k\}_{k=1}^{K}$. Within each component, tokens evolve according to
\begin{equation}
\dot{x}_i(t)
=
\mathbf P_{x_i(t)}^\perp
\left(
    \frac{1}{d_i^{(R)}}
    \sum_{j\in\mathcal N_i^{(R)}}
    e^{\beta\langle x_i(t),x_j(t)\rangle}
    x_j(t)
\right),
\qquad
i\in\mathcal I_k.
\label{eq:neighborhood_ode}
\end{equation}
Here, $\mathcal N_i^{(R)}\subseteq\mathcal I_k$ is the interaction neighborhood of token $i$, and $d_i^{(R)}=\lvert\mathcal N_i^{(R)}\rvert$. Each component therefore evolves independently, with its consensus direction determined by the coupled evolution of its tokens. These emergent consensus directions serve as implicit opinion leaders.


Our contributions combine convergence guarantees for both mechanisms with empirical evidence from numerical simulations and frontier sparse-attention LLMs:

\begin{enumerate}[leftmargin=*]
    \item \textbf{Explicit opinion leader dynamics: convergence under fixed representatives.} We formulate token evolution under fixed representatives as a \emph{reverse Wasserstein gradient flow}. We establish convergence of token trajectories to critical points of the induced potential for arbitrary initialization. Under geometric separation and dominance conditions, we further prove exponential convergence of each group to a distinct attracting local maximum.

    \item \textbf{Implicit opinion leader dynamics: consensus under restricted interactions.} We analyze token evolution on a fixed sparse graph. We establish a \emph{reverse Wasserstein gradient flow} formulation in a degree-weighted product space for each connected component. Under suitable graph and initialization conditions, the components converge exponentially to distinct consensus directions. The convergence-rate bound depends explicitly on the normalized Laplacian spectral gap, linking internal connectivity to the rate of token alignment.


    \item \textbf{Empirical validation of predicted clustering behavior in frontier LLMs.} We analyze token evolution in \textcolor{explicitcolor}{\textbf{Kimi-K3}}~\citep{team2026kimi}, \textcolor{implicitcolor}{\textbf{MiniMax-M3}}~\citep{lai2026minimax}, and \textcolor{explicitcolor}{\textbf{DeepSeek-}}\textcolor{implicitcolor}{\textbf{V4-Flash}}~\citep{xu2026deepseek}, motivating \textcolor{explicitcolor}{\textbf{Explicit}}, \textcolor{implicitcolor}{\textbf{Implicit}}, and \textcolor{explicitcolor}{\textbf{Hyb}}\textcolor{implicitcolor}{\textbf{rid}} interpretations, respectively. Across four benchmarks, all three exhibit clearer cluster separation and higher clustering scores in projected representations than the dense-attention model \textcolor{standardcolor}{\textbf{GLM-4.7-Flash}}~\citep{zeng2025glm}, consistent with the multiple-group structure predicted by our theory. Finite-particle simulations further illustrate the theoretical convergence behavior.

\end{enumerate}

%% file: 2_opinion_leader_dynamics.tex
\section{Opinion Leader Dynamics}
\label{sec:leader_induced_dynamics}




\subsection{Explicit Opinion Leader Dynamics}
\label{subsec:fixed_leader_dynamics}

The explicit model describes token evolution under a potential induced by fixed representatives. We first construct the representatives from initial token groups, then characterize the resulting dynamics for arbitrary initialization and under stronger conditions ensuring separated attractors.

Let $\{\mathcal I_a\}_{a=1}^{m}$ be a partition of $\{1,\ldots,n\}$ into nonempty groups, with initial token states $x_i(0)\in\mathbb S^{d-1}$. For each group, define its size, relative mass, and initial empirical distribution:
\begin{equation}
    n_a
    =
    |\mathcal I_a|,
    \qquad
    \omega_a
    =
    \frac{n_a}{n},
    \qquad
    \mu_0^a
    =
    \frac{1}{n_a}
    \sum_{i\in\mathcal I_a}
    \delta_{x_i(0)}.
    \label{eq:compressed_group_initial_measure}
\end{equation}
The initial token distribution is therefore
\begin{equation}
    \mu_0
    =
    \sum_{a=1}^{m}
    \omega_a\mu_0^a.
    \label{eq:compressed_group_decomposition}
\end{equation}
A compression map $\mathsf C_\theta$ assigns each group a nonzero representative:
\begin{equation}
    z_a
    =
    \mathsf C_\theta
    \left(
        (x_i(0))_{i\in\mathcal I_a}
    \right)
    \neq
    0,
    \qquad
    r_a
    =
    \lVert z_a\rVert_2,
    \qquad
    \ell_a
    =
    \frac{z_a}{r_a}
    \in
    \mathbb S^{d-1}.
    \label{eq:compressed_representative_construction}
\end{equation}
The representatives remain fixed throughout the subsequent evolution. More generally, we use the measure-valued formulation in Eq.~\ref{eq:compressed_group_decomposition}, allowing arbitrary $\mu_0^a\in\mathcal P(\mathbb S^{d-1})$, with $\omega_a>0$ and $\sum_{a=1}^{m}\omega_a=1$. Each initial group distribution is associated with a fixed nonzero representative $z_a$.

Define the representative measure
\begin{equation}
    \nu
    =
    \sum_{a=1}^{m}
    \rho_a\delta_{z_a},
    \qquad
    \rho_a>0,
    \qquad
    \sum_{a=1}^{m}\rho_a=1.
    \label{eq:compressed_representative_measure}
\end{equation}

The weights $\omega_a$ and $\rho_a$ play different roles: $\omega_a$ specifies the mass of token group $a$, whereas $\rho_a$ controls the contribution of its representative to the interaction field. The representatives are computed at initialization and remain fixed throughout the dynamics. For $\beta>0$, the time-independent measure $\nu$ induces the potential $\Phi_\nu$ and its spherical gradient field $\mathcal X_{\nu,\beta}$:

\begin{equation}
    \begin{aligned}
        \Phi_\nu(x)
        &=
        \frac{1}{\beta}
        \int_{\mathbb R^d}
        e^{\beta\langle x,z\rangle}
        d\nu(z)
        =
        \frac{1}{\beta}
        \sum_{a=1}^{m}
        \rho_a
        e^{\beta r_a\langle x,\ell_a\rangle},
        \\
        \mathcal X_{\nu,\beta}(x)
        &=
        \nabla^\circ\Phi_\nu(x)
        =
        \sum_{a=1}^{m}
        \rho_a r_a
        e^{\beta r_a\langle x,\ell_a\rangle}
        \mathbf P_x^\perp(\ell_a).
    \end{aligned}
    \label{eq:fixed_leader_potential_and_field}
\end{equation}
Here, $\nabla^\circ$ denotes the spherical gradient. All tokens evolve under this shared field, so the attracting directions are determined jointly by all representatives. The smooth field $\mathcal X_{\nu,\beta}$ generates a global flow $\phi_t$ on $\mathbb S^{d-1}$. Define the evolving group distributions and their mixture by
\begin{equation}
    \mu_t^a
    =
    (\phi_t)_\#\mu_0^a,
    \qquad
    \mu_t
    =
    \sum_{a=1}^{m}
    \omega_a\mu_t^a
    =
    (\phi_t)_\#\mu_0.
    \label{eq:explicit_groupwise_evolution}
\end{equation}
For both empirical and general initial distributions, $\mu_t$ satisfies the continuity equation
\begin{equation}
    \partial_t\mu_t
    +
    \operatorname{div}^\circ
    \left(
        \mu_t\mathcal X_{\nu,\beta}
    \right)
    =
    0.
    \label{eq:fixed_leader_transport}
\end{equation}
where $\operatorname{div}^\circ$ denotes the spherical divergence. Consider the energy
\begin{equation}
    \mathbf E_\beta[\mu|\nu]
    =
    \int_{\mathbb S^{d-1}}
    \Phi_\nu(x)
    \,d\mu(x).
    \label{eq:fixed_leader_energy}
\end{equation}
Along Eq.~\ref{eq:fixed_leader_transport}, it satisfies $\frac{d}{dt}\mathbf E_\beta[\mu_t|\nu]=\int_{\mathbb S^{d-1}}\lVert\mathcal X_{\nu,\beta}(x)\rVert_2^2\,d\mu_t(x)\geq0$. Thus, individual tokens follow spherical gradient ascent on $\Phi_\nu$, while their distribution evolves as the \emph{reverse Wasserstein gradient flow} of $\mathbf E_\beta[\cdot\mid\nu]$. Appendix~\ref{app:explicit_gradient_flow} provides the derivation.

We first establish asymptotic convergence without imposing geometric conditions on the initial groups. Let $\operatorname{Crit}(\Phi_\nu)$ denote the critical set of $\Phi_\nu$, and let $W_2$ be the Wasserstein distance induced by the spherical geodesic distance $d_{\mathbb S}$. Additional notation is summarized in Table~\ref{tab:notation_explicit} of Appendix.

\begin{corollary}[\textbf{Global asymptotic convergence}]
\label{cor:explicit_global_convergence}
For every $x\in\mathbb S^{d-1}$, the limit
$\phi_\infty(x)=\lim_{t\to\infty}\phi_t(x)$ exists and belongs to $\operatorname{Crit}(\Phi_\nu)$. Define
$\mu_\infty^a=(\phi_\infty)_\#\mu_0^a$ for each $a$, and
$\mu_\infty=\sum_{a=1}^{m}\omega_a\mu_\infty^a
=(\phi_\infty)_\#\mu_0$.
Then, as $t\to\infty$,
\begin{equation}
    W_2(\mu_t^a,\mu_\infty^a)
    \longrightarrow 0
    \quad\text{for every }a,
    \qquad
    W_2(\mu_t,\mu_\infty)
    \longrightarrow 0.
    \label{eq:explicit_global_groupwise_convergence}
\end{equation}
Moreover,
\begin{equation}
    \operatorname{supp}(\mu_\infty^a)
    \subseteq
    \operatorname{Crit}(\Phi_\nu)
    \quad\text{for every }a,
    \qquad
    \operatorname{supp}(\mu_\infty)
    \subseteq
    \operatorname{Crit}(\Phi_\nu).
\end{equation}
\end{corollary}

Corollary~\ref{cor:explicit_global_convergence} guarantees convergence of every token trajectory, while allowing each group to reach multiple critical points and tokens from different groups to share the same limit. The proof uses the analyticity of $\Phi_\nu$ and is given in Appendix~\ref{app:explicit_global_convergence}.

To obtain one attracting direction per group and an explicit exponential rate, we next impose concentration, separation, and representative-dominance conditions. Let
$M_0^a=\int_{\mathbb S^{d-1}}x\,d\mu_0^a(x)$
be the initial mean of group $a$. We write
$S_\theta^+(v)=\{x\in\mathbb S^{d-1}:
\langle x,v\rangle\geq\cos\theta\}$
for the closed spherical cap of radius $\theta$ centered at $v$.

\begin{assumption}[\textbf{Geometric separation and dominance}]
\label{ass:explicit_separated_basins}
Assume $M_0^a\neq0$ for every $a$, and define
$c_a=M_0^a/\lVert M_0^a\rVert_2$.
Suppose that $m\geq2$ and that there exist
$\delta,\Delta,\eta>0$ and $\alpha\in(0,\pi/2)$ such that
\begin{equation}
\begin{aligned}
    \operatorname{supp}(\mu_0^a)
    &\subseteq S_\delta^+(c_a),
    \qquad
    d_{\mathbb S}(\ell_a,c_a)\leq\eta,
    \qquad&&
    a=1,\ldots,m,
    \\
    d_{\mathbb S}(c_a,c_b)
    &\geq\Delta,
    &&
    1\leq a<b\leq m,
\end{aligned}
\end{equation}
with
\begin{equation}
    \delta+\eta<\alpha,
    \qquad
    2(\alpha+\eta)<\Delta.
    \label{eq:explicit_cap_separation}
\end{equation}
Assume further that
\begin{equation}
    \mathbf I_a(\alpha)>0,
    \qquad
    \mathbf H_a(\alpha)>0,
    \qquad
    a=1,\ldots,m,
    \label{eq:explicit_basin_dominance}
\end{equation}
where $\mathbf I_a(\alpha)$ and $\mathbf H_a(\alpha)$ are defined in Table~\ref{tab:notation_explicit}.
\end{assumption}

The geometric conditions place each initial group inside a cap around its representative and keep these caps separated. The dominance conditions ensure that the field points inward at each cap boundary and that the potential is uniformly strictly concave within the cap. Together, they yield a unique attracting direction for each group.

\begin{theorem}[\textbf{Exponential convergence to separated attractors}]
\label{thm:explicit_separated_attractors}
Under Assumption~\ref{ass:explicit_separated_basins}, the caps
$S_\alpha^+(\ell_a)$ are pairwise disjoint and forward invariant. Each cap contains a unique critical point $u_a$ of $\Phi_\nu$, which is an interior nondegenerate attracting local maximizer. If $B_a$ denotes its attraction basin, then
\begin{equation}
    \operatorname{supp}(\mu_0^a)
    \subseteq
    S_{\delta+\eta}^+(\ell_a)
    \Subset
    \operatorname{int}S_\alpha^+(\ell_a)
    \subseteq
    B_a.
    \label{eq:explicit_support_in_basins}
\end{equation}

The attractor locations satisfy
\begin{equation}
    d_{\mathbb S}(u_a,\ell_a)
    \leq\alpha,
    \qquad
    d_{\mathbb S}(u_a,c_a)
    \leq\alpha+\eta.
    \label{eq:explicit_attractor_location}
\end{equation}
For $a\neq b$, their separation is bounded below by
\begin{align}
    d_{\mathbb S}(u_a,u_b)
    &\geq
    d_{\mathbb S}(c_a,c_b)-2(\alpha+\eta)
    \geq
    \Delta-2\eta-2\alpha
    >0.
    \label{eq:explicit_maximizer_separation}
\end{align}

For each $a$, define
$D_{0,a}^2=\int_{\mathbb S^{d-1}}
d_{\mathbb S}(x,u_a)^2\,d\mu_0^a(x)$.
Then
\begin{equation}
    \operatorname{supp}(\mu_t^a)
    \subseteq S_\alpha^+(\ell_a),
    \qquad
    W_2(\mu_t^a,\delta_{u_a})
    \leq
    D_{0,a}e^{-\mathbf H_a(\alpha)t},
    \qquad
    t\geq0.
    \label{eq:explicit_groupwise_rate}
\end{equation}

Define
\begin{equation}
    \mathbf H_{\min}(\alpha)
    =
    \min_{1\leq a\leq m}\mathbf H_a(\alpha),
    \qquad
    \mu_\infty
    =
    \sum_{a=1}^{m}\omega_a\delta_{u_a},
    \qquad
    D_0^2
    =
    \sum_{a=1}^{m}\omega_aD_{0,a}^2.
    \label{eq:explicit_limit_and_initial_distance}
\end{equation}
The limiting distribution has exactly $m$ support points:
$\operatorname{supp}(\mu_\infty)=\{u_1,\ldots,u_m\}$.
Moreover,
\begin{equation}
    W_2(\mu_t,\mu_\infty)
    \leq
    D_0e^{-\mathbf H_{\min}(\alpha)t}
    \leq
    (\alpha+\delta+\eta)
    e^{-\mathbf H_{\min}(\alpha)t},
    \qquad
    t\geq0.
    \label{eq:explicit_quantitative_rate}
\end{equation}
\end{theorem}

Theorem~\ref{thm:explicit_separated_attractors} shows that each group contracts toward its own \emph{explicit opinion leader} $u_a$ while remaining separated from the other groups. The bound $\mathbf H_a(\alpha)$ quantifies the local curvature driving this contraction, and the slowest group determines the guaranteed rate for the full distribution. Each leader is a local maximizer of the potential generated by all representatives, rather than necessarily the representative direction $\ell_a$ itself. The proof is given in Appendix~\ref{app:explicit_separated_attractors_proof}.

\subsection{Implicit Opinion Leader Dynamics}
\label{subsec:neighborhood_dynamics}

The implicit model describes alignment through mutual token interactions on a fixed sparse graph. We first identify the connected components, then establish consensus within each component under suitable initial geometry. The limiting directions emerge from the coupled dynamics of the tokens.

Let $(\mathcal Z,d_{\mathcal Z})$ be a metric space and associate each token $i$ with a fixed interaction coordinate $p_i\in\mathcal Z$. These coordinates determine which tokens interact, while the token states evolve on $\mathbb S^{d-1}$. For an interaction radius $R>0$, define the neighborhoods and adjacency entries by
\begin{equation}
    \mathcal N_i^{(R)}
    =
    \left\{
        j\in\{1,\ldots,n\}:
        d_{\mathcal Z}(p_i,p_j)\leq R
    \right\},
    \qquad
    A_{ij}^{(R)}
    =
    \mathbf 1
    \left\{
        j\in\mathcal N_i^{(R)}
    \right\},
\label{eq:neighborhood_radius_graph}
\end{equation}
Set $d_i^{(R)}=|\mathcal N_i^{(R)}|$. Symmetry of the metric and the identity $d_{\mathcal Z}(p_i,p_i)=0$ ensure that $A^{(R)}$ defines an undirected graph $G_R$ with a self-loop at every node.

We write $i\sim_R j$ if tokens $i$ and $j$ are connected by a path in $G_R$. Equivalently, there exist an integer $s\geq0$ and a sequence $i_0,\ldots,i_s$ such that
\begin{equation}
    i_0=i,
    \qquad
    i_s=j,
    \qquad
    d_{\mathcal Z}(p_{i_{r-1}},p_{i_r})\leq R,
    \quad
    r=1,\ldots,s.
\label{eq:neighborhood_reachability}
\end{equation}
We call this relation \emph{$R$-reachability}. To obtain multiple interaction groups, each containing at least two tokens, we impose the following condition on $R$.

\begin{assumption}[\textbf{Separated interaction scales}]
\label{ass:neighborhood_mask}
Let $\mathbb T_n$ denote the set of spanning trees on $\{1,\ldots,n\}$, and let $\mathcal E(T)$ be the edge set of $T\in\mathbb T_n$. Define
\begin{equation}
    \begin{aligned}
        R_{\mathrm{loc}}
        =
        \max_{i\in\{1,\ldots,n\}}
        \min_{j\in\{1,\ldots,n\}\setminus\{i\}}
        d_{\mathcal Z}(p_i,p_j),
        \quad
        R_{\mathrm{conn}}
        =
        \min_{T\in\mathbb T_n}
        \max_{\{i,j\}\in\mathcal E(T)}
        d_{\mathcal Z}(p_i,p_j).
    \end{aligned}
    \label{eq:neighborhood_interaction_scales}
\end{equation}
The interaction radius satisfies
$R_{\mathrm{loc}}\leq R<R_{\mathrm{conn}}$.
\end{assumption}
The lower bound $R_{\mathrm{loc}}$ ensures that every token has a neighbor other than itself. The upper bound $R_{\mathrm{conn}}$ keeps $R$ below the smallest radius at which $G_R$ becomes connected.

For the reachability relation $\sim_R$, the equivalence classes $\{\mathcal I_k\}_{k=1}^{K}$ are precisely the connected components of $G_R$, giving the partition
\begin{equation}
    \{1,\ldots,n\}
    =
    \bigsqcup_{k=1}^{K}\mathcal I_k,
    \qquad
    n_k=|\mathcal I_k|,
    \qquad
    \sum_{k=1}^{K}n_k=n.
\label{eq:neighborhood_coordinate_partition}
\end{equation}

Assumption~\ref{ass:neighborhood_mask} guarantees $K\geq2$ and $n_k\geq2$ for every $k$. Since $\mathcal N_i^{(R)}\subseteq\mathcal I_k$ for $i\in\mathcal I_k$, Eq.~\ref{eq:neighborhood_ode} decomposes into $K$ independent particle systems on the connected subgraphs $G_R[\mathcal I_k]$. These subgraphs need not be complete: the analysis retains the sparse interactions within each component. Appendix~\ref{app:neighborhood_component_decomposition} provides the full proof.

For each component, let $\mathbf A_k$ and $\mathbf D_k$ denote its adjacency and degree matrices. The associated random walk has transition probabilities $\mathsf P_{ij}^k=A_{ij}^{(R)}/d_i^{(R)}$ and stationary weights $\pi_i^k=d_i^{(R)}/\operatorname{vol}_k$, where $\operatorname{vol}_k=\sum_{i\in\mathcal I_k}d_i^{(R)}$. Symmetry of $\mathbf A_k$ gives the detailed-balance identity $\pi_i^k\mathsf P_{ij}^k=\pi_j^k\mathsf P_{ji}^k$. Connectivity ensures that the normalized Laplacian gap $\lambda_k=\lambda_2(\mathbf I_{n_k}-\mathbf D_k^{-1/2}\mathbf A_k\mathbf D_k^{-1/2})$ is positive, where $\lambda_2$ denotes the second-smallest eigenvalue. Table~\ref{tab:notation_implicit} summarizes the notation.

To express the dynamics in a Wasserstein framework while preserving the graph structure, we associate each node $i\in\mathcal I_k$ with a measure $\nu_{i,t}^k\in\mathcal P_2(\mathbb S^{d-1})$ and write $\boldsymbol\nu_t^k=(\nu_{i,t}^k)_{i\in\mathcal I_k}$. The velocity field and continuity equation are
\begin{equation}
    \mathcal X_{i,\boldsymbol\nu^k,\beta}(x)
    =
    \sum_{j\in\mathcal I_k}
    \mathsf P_{ij}^k
    \int_{\mathbb S^{d-1}}
    \mathbf P_x^\perp(y)
    e^{\beta\langle x,y\rangle}
    \,d\nu_j^k(y),
    \
    \partial_t\nu_{i,t}^k
    +
    \operatorname{div}^{\circ}
    \left(
        \nu_{i,t}^k
        \mathcal X_{i,\boldsymbol\nu_t^k,\beta}
    \right)
    =
    0,
    \
    i\in\mathcal I_k.
\label{eq:neighborhood_product_continuity_equation}
\end{equation}
Choosing Dirac initial measures yields
$\nu_{i,t}^k=\delta_{x_i(t)}$ and recovers the particle dynamics in Eq.~\ref{eq:neighborhood_ode}.

The stationary weights induce the degree-weighted product Wasserstein metric
\begin{equation}
    \mathcal W_{2,\pi^k}^{2}
    \left(
        \boldsymbol\nu^k,
        \boldsymbol\eta^k
    \right)
    =
    \sum_{i\in\mathcal I_k}
    \pi_i^k
    W_2^2
    \left(
        \nu_i^k,
        \eta_i^k
    \right).
\label{eq:neighborhood_product_wasserstein}
\end{equation}
Smoothness of the interaction kernel and compactness of the sphere ensure a unique global characteristic solution for every initial tuple. This solution is locally absolutely continuous in $\mathcal W_{2,\pi^k}$, as established in Lemma~\ref{lem:neighborhood_wellposedness}. Detailed balance yields a graph interaction energy whose \emph{reverse Wasserstein gradient flow} is Eq.~\ref{eq:neighborhood_product_continuity_equation} in this weighted product space. Appendix~\ref{app:neighborhood_wasserstein_structure} gives the energy and its derivation. We use two aggregate measures to relate this formulation to token distributions:
\begin{equation}
    \mu_t^k
    =
    \frac{1}{n_k}
    \sum_{i\in\mathcal I_k}
    \nu_{i,t}^k,
    \qquad
    \widehat{\mu}_t^k
    =
    \sum_{i\in\mathcal I_k}
    \pi_i^k\nu_{i,t}^k,
    \qquad
    m_k
    =
    \frac{n_k}{n},
    \qquad
    \mu_t
    =
    \sum_{k=1}^{K}
    m_k\mu_t^k.
\label{eq:neighborhood_aggregate_measures}
\end{equation}
The uniform aggregate $\mu_t^k$ assigns equal weight to each token, while $\widehat{\mu}_t^k$ uses the degree weights natural to the graph dynamics. We next impose concentration within components and separation between them. Let
$\Omega_k=\bigcup_{i\in\mathcal I_k}\operatorname{supp}(\nu_{i,0}^k)$
denote the union of initial supports in component $k$.

\begin{assumption}[\textbf{Concentrated and separated initial states}]
\label{ass:neighborhood_initial_states}
There exist $\alpha\in[0,\pi/4)$ and directions
$u_1,\ldots,u_K\in\mathbb S^{d-1}$ such that
\begin{equation}
    \Omega_k
    \subseteq
    S_\alpha^+(u_k),
    \quad
    k=1,\ldots,K,
    \qquad
    d_{\mathbb S}(u_k,u_\ell)>2\alpha,
    \quad
    1\leq k<\ell\leq K.
\end{equation}
The directions $u_k$ specify the centers of the initial caps; they are not prescribed consensus directions.
\end{assumption}
This assumption places each component inside a spherical cap, with a positive gap between caps. The proof shows that these caps remain invariant and that disagreement within each component decays exponentially. The convergence bounds use the initial degree-weighted disagreement $V_0^k$ and the degree-imbalance factor $\chi_k$, which accounts for the difference between uniform and degree-weighted aggregates. Their definitions, together with the constants $\kappa_k$, $C_k$, $\kappa_{\min}$, and $C_G$ used below, are collected in Table~\ref{tab:notation_implicit}.

\begin{theorem}[\textbf{Componentwise exponential consensus}]
\label{thm:neighborhood_multicluster}
Under Assumptions~\ref{ass:neighborhood_mask} and~\ref{ass:neighborhood_initial_states}, each component remains in its initial cap:
\begin{equation}
    \operatorname{supp}(\nu_{i,t}^k)
    \subseteq
    S_\alpha^+(u_k),
    \qquad
    t\geq0,
    \quad
    k=1,\ldots,K,
    \quad
    i\in\mathcal I_k.
    \label{eq:neighborhood_cap_invariance_main}
\end{equation}
There exist pairwise distinct directions
$x_\infty^k\in S_\alpha^+(u_k)$ such that, with
$\boldsymbol\nu_\infty^k=(\delta_{x_\infty^k})_{i\in\mathcal I_k}$,
\begin{equation}
    \mathcal W_{2,\pi^k}
    \left(
        \boldsymbol\nu_t^k,
        \boldsymbol\nu_\infty^k
    \right)
    \leq
    C_k(V_0^k)^{1/2}e^{-\kappa_k t},
    \qquad
    t\geq0.
    \label{eq:neighborhood_product_rate_main}
\end{equation}
Consequently, the full token distribution converges to a mixture of $K$ distinct point masses:
\begin{equation}
    \mu_\infty
    =
    \sum_{k=1}^{K}m_k\delta_{x_\infty^k},
    \qquad
    W_2(\mu_t,\mu_\infty)
    \leq
    C_Ge^{-\kappa_{\min}t},
    \qquad
    t\geq0.
    \label{eq:neighborhood_global_rate_main}
\end{equation}
\end{theorem}

Theorem~\ref{thm:neighborhood_multicluster} establishes internal alignment within each component while preserving separation across components. The limiting directions $x_\infty^k$ are the \emph{implicit opinion leaders}, determined by the initial states and interactions within their respective groups. The graph structure enters the convergence bound through
$\kappa_k=\lambda_k\cos(2\alpha)e^{\beta\cos(2\alpha)}$.
For fixed $\alpha$ and $\beta$, this guaranteed rate scales as
$\Theta(n_k^{-2})$ for self-looped path graphs and as $\Theta(1)$ for graph families with a uniformly positive normalized Laplacian gap. Thus, the theorem quantifies how connectivity within a group affects its guaranteed rate of consensus. The full proof is provided in Appendix~\ref{app:neighborhood_proof}.


%% file: 3_experiments.tex
\begin{figure}[t]
    \centering
    \includegraphics[width=.99\linewidth]{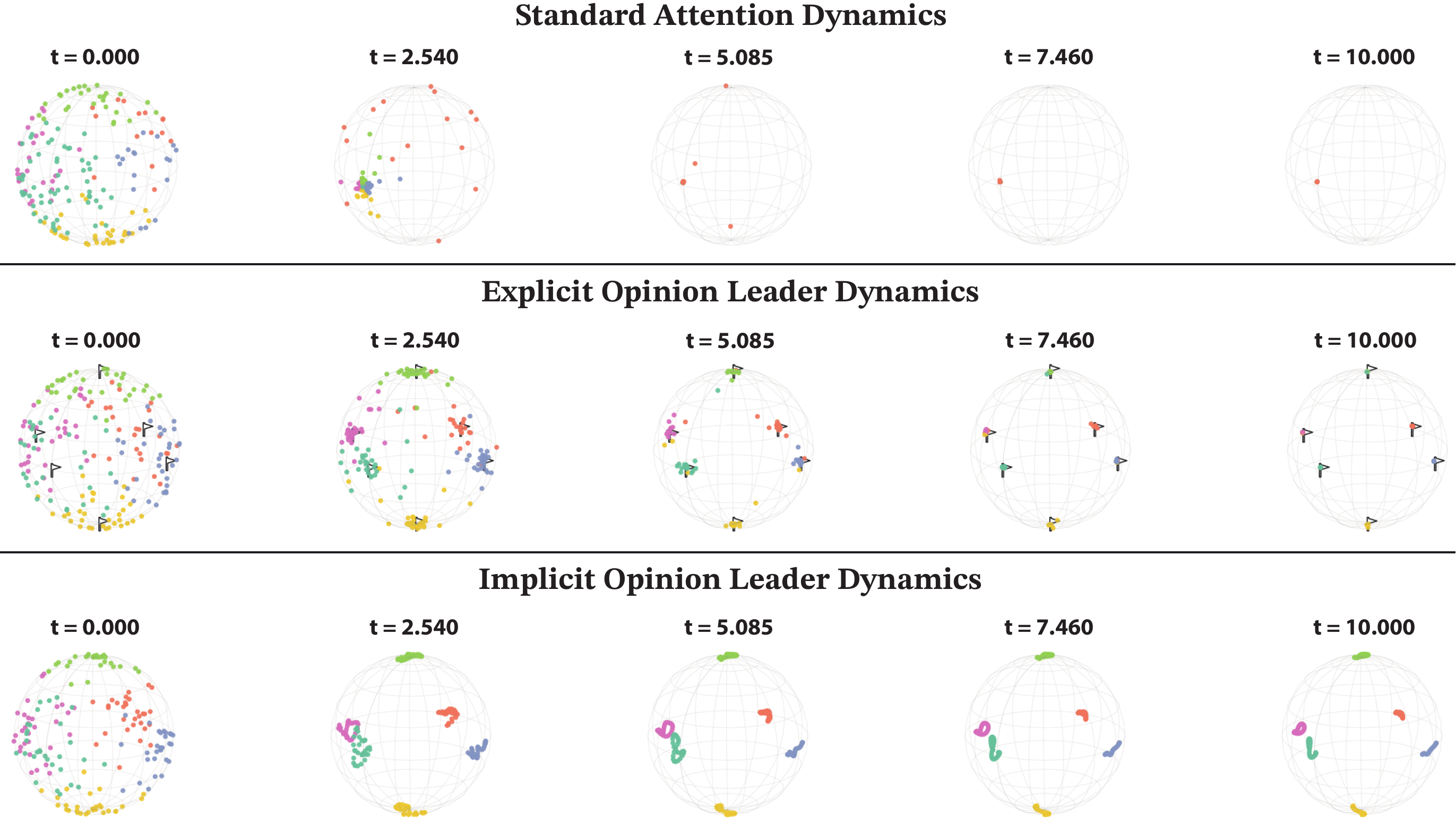}
    \caption{Finite-particle simulations on $\mathbb S^2$ of Standard Attention dynamics (\emph{top}), explicit opinion leader dynamics under general initialization (\emph{middle}), and implicit opinion leader dynamics (\emph{bottom}). Columns show five snapshots over $t\in[0,10]$.  Colors are assigned at $t=0$ and remain fixed.} 
    \label{fig:attention_dynamics_simulation}
    \vspace{-3mm}
\end{figure}

\section{Experiments}
\subsection{Experimental Setup}


\noindent\textbf{Numerical simulations.} We simulate Standard Attention dynamics, implicit opinion leader dynamics, and explicit opinion leader dynamics under both general initialization (Corollary~\ref{cor:explicit_global_convergence}) and initialization satisfying the geometric conditions of Theorem~\ref{thm:explicit_separated_attractors}. Implementation details are provided in Appendix~\ref{app:simulation_implementation}. Additional simulations of explicit opinion leader dynamics under the Theorem~\ref{thm:explicit_separated_attractors}'s initialization conditions appear in Appendix~\ref{app:explicit_conditional_simulation}.

\noindent\textbf{Frontier LLM analysis.} We track layerwise evolution of token representations in three frontier sparse-attention LLMs: \textcolor{explicitcolor}{\textbf{Kimi-K3}}~\citep{team2026kimi}, \textcolor{implicitcolor}{\textbf{MiniMax-M3}}~\citep{lai2026minimax}, and \textcolor{explicitcolor}{\textbf{DeepSeek-}}\textcolor{implicitcolor}{\textbf{V4-Flash}}~\citep{xu2026deepseek}. We also use \textcolor{standardcolor}{\textbf{GLM-4.7-Flash}}~\citep{zeng2025glm}, which employs standard dense attention for comparison. The analysis covers inputs from four benchmarks: HumanEval~\citep{chen2021evaluating}, HellaSwag~\citep{zellers2019hellaswag}, ARC-Easy~\citep{clark2018think}, and MATH~\citep{hendrycks2021measuring}. These benchmarks provide diverse inputs for examining representation dynamics.

We obtain low-dimensional embeddings using UMAP~\citep{mcinnes2018umap}, and its spherical variants, with spherical visualizations displaying token representations on $\mathbb S^2$. We apply HDBSCAN~\citep{mcinnes2017hdbscan} to the reduced representations to identify token clusters. The resulting cluster labels determine token colors, which remain fixed across layers so that the earlier representations of each final-layer group can be inspected. Extended results are provided in Appendix~\ref{app:extended_llm_analysis}.

\noindent\textbf{Measuring cluster separation.} We use the \emph{cosine silhouette score}~\citep{rousseeuw1987silhouettes} to quantify \textbf{\emph{within-cluster cohesion}} and \textbf{\emph{between-cluster separation}}. For each token, the score compares its average cosine distance to tokens in its own cluster with that to the nearest competing cluster. Scores range in $[-1,1]$: higher values indicate more cohesive and better-separated clusters, values near zero indicate weak separation, and negative values indicate greater proximity to another cluster. We average the tokenwise scores to obtain an overall measure of clustering quality; the detailed computation is provided in Appendix~\ref{app:s_cos}. All scores are computed in the reduced representation space. For reference, we apply the same analysis to standard Gaussian vectors with matched dimensions.

\subsection{Numerical Simulations}


We compare the three dynamics under the initialization settings described below. For Standard Attention, particles are sampled independently from a von Mises--Fisher distribution on $\mathbb S^2$. For \emph{explicit opinion leader}, particles are initialized uniformly on $\mathbb S^2$ and evolve under $m=6$ fixed representatives. For \emph{implicit opinion leader}, the $n=180$ particles are divided equally into $K=6$ disconnected components, with initial states satisfying our concentration and separation conditions.

Fig.~\ref{fig:attention_dynamics_simulation} illustrates the resulting contrast. Standard Attention progressively aligns particles toward a single consensus direction. Explicit opinion leader dynamics organize initially dispersed particles into multiple clusters around attractors of the representative-induced potential, whereas implicit opinion leader dynamics align particles within each connected component while maintaining separation. The standard and implicit simulations illustrate the consensus behavior characterized by Theorems~\ref{thm:sa_general_consensus} and~\ref{thm:neighborhood_multicluster}. The explicit simulation illustrates convergence to critical points as established by Corollary~\ref{cor:explicit_global_convergence}, with the chosen representatives yielding several attracting clusters.




\begin{table}[htbp]
\vspace{-3mm}
  \caption{
    Cosine silhouette scores across LLMs and benchmarks. For each model, the \emph{upper} row reports mean final-layer scores across $100$ random samples after projecting hidden states with spherical UMAP, while the \emph{lower} row gives the Gaussian reference score at the same hidden dimension.}
  \label{tab:cosine_silhouette}
  \centering
  \small
  \setlength{\tabcolsep}{3pt}

  \begin{NiceTabularX}{\linewidth}{
    @{}l *{4}{>{\centering\arraybackslash}X}@{}
  }
    \toprule
    \textbf{Model}
    & \textbf{HumanEval}
    & \textbf{ARC-Easy}
    & \textbf{HellaSwag}
    & \textbf{MATH} \\
    \midrule

    & \textbf{0.77 $\pm$ 0.22} ($\uparrow$ \textbf{0.60})
    & \textbf{0.87 $\pm$ 0.15} ($\uparrow$ \textbf{0.68})
    & \textbf{0.77 $\pm$ 0.17} ($\uparrow$ \textbf{0.62})
    & \textbf{0.86 $\pm$ 0.18} ($\uparrow$ \textbf{0.67}) \\

    \multirow{-2}{*}{\textcolor{explicitcolor}{\textbf{Kimi-K3}}}
    & 0.17 $\pm$ 0.04
    & 0.19 $\pm$ 0.05
    & 0.15 $\pm$ 0.05
    & 0.19 $\pm$ 0.04 \\
    \midrule

    & \textbf{0.87 $\pm$ 0.07} ($\uparrow$ \textbf{0.71})
    & \textbf{0.84 $\pm$ 0.07} ($\uparrow$ \textbf{0.67})
    & \textbf{0.85 $\pm$ 0.09} ($\uparrow$ \textbf{0.69})
    & \textbf{0.83 $\pm$ 0.11} ($\uparrow$ \textbf{0.66}) \\

    \multirow{-2}{*}{\textcolor{implicitcolor}{\textbf{MiniMax-M3}}}
    & 0.16 $\pm$ 0.04
    & 0.17 $\pm$ 0.04
    & 0.16 $\pm$ 0.04
    & 0.17 $\pm$ 0.04 \\
    \midrule

    & \textbf{0.71 $\pm$ 0.17} ($\uparrow$ \textbf{0.52})
    & \textbf{0.85 $\pm$ 0.12} ($\uparrow$ \textbf{0.63})
    & \textbf{0.64 $\pm$ 0.19} ($\uparrow$ \textbf{0.47})
    & \textbf{0.84 $\pm$ 0.12} ($\uparrow$ \textbf{0.63}) \\

    \multirow{-2}{*}{\textcolor{explicitcolor}{\textbf{DeepSeek-}}\textcolor{implicitcolor}{\textbf{V4-Flash}}}
    & 0.19 $\pm$ 0.06
    & 0.22 $\pm$ 0.12
    & 0.17 $\pm$ 0.04
    & 0.21 $\pm$ 0.13 \\
    \midrule

    & 0.35 $\pm$ 0.27 ($\uparrow$ 0.19)
    & 0.24 $\pm$ 0.17 ($\uparrow$ 0.08)
    & 0.34 $\pm$ 0.26 ($\uparrow$ 0.19)
    & 0.30 $\pm$ 0.20 ($\uparrow$ 0.13) \\

    \multirow{-2}{*}{\textcolor{standardcolor}{\textbf{GLM-4.7-Flash}}}
    & 0.16 $\pm$ 0.06
    & 0.16 $\pm$ 0.13
    & 0.15 $\pm$ 0.04
    & 0.17 $\pm$ 0.12 \\
    \bottomrule
  \end{NiceTabularX}
  \vspace{-3mm}
\end{table}

\subsection{Frontier LLM Analysis}

\begin{figure}[t]
    \centering
    \includegraphics[width=.99\linewidth]{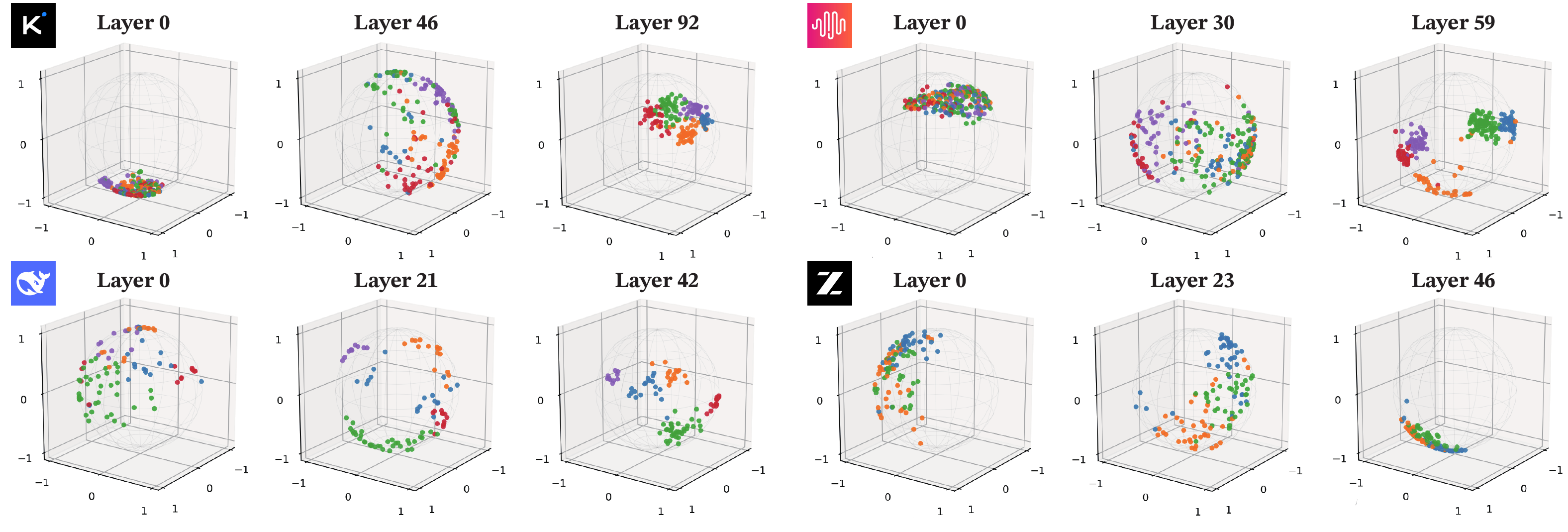}
    \caption{Layerwise evolution of token hidden states for ARC-Easy samples in \textcolor{explicitcolor}{\textbf{Kimi-K3}}, \textcolor{implicitcolor}{\textbf{MiniMax-M3}}, \textcolor{explicitcolor}{\textbf{DeepSeek-}}\textcolor{implicitcolor}{\textbf{V4-Flash}}, and \textcolor{standardcolor}{\textbf{GLM-4.7-Flash}}, visualized on unit sphere using spherical UMAP.}
    \vspace{-3mm}
    \label{fig:llm_observation}
\end{figure}

\noindent\textbf{Layerwise token dynamics.}
Fig.~\ref{fig:llm_observation} visualizes token representations for ARC-Easy samples using spherical UMAP. \textcolor{explicitcolor}{\textbf{Kimi-K3}}, \textcolor{implicitcolor}{\textbf{MiniMax-M3}}, and \textcolor{explicitcolor}{\textbf{DeepSeek-}}\textcolor{implicitcolor}{\textbf{V4-Flash}} exhibit distinct final-layer clusters, whereas \textcolor{standardcolor}{\textbf{GLM-4.7-Flash}} shows weaker intermediate separation followed by a more concentrated final configuration. These observations are qualitatively consistent with the contrast between separated token groups and global alignment studied in our framework.

\noindent\textbf{Quantitative cluster separation.}
Table~\ref{tab:cosine_silhouette} reports final-layer cosine silhouette scores over $100$ samples per benchmark. Across all four benchmarks, the three sparse-attention LLMs score $\mathbf{0.64}$--$\mathbf{0.87}$, compared with $0.24$--$0.35$ for \textcolor{standardcolor}{\textbf{GLM-4.7-Flash}}. Their gains over matched Gaussian references are also larger: $\mathbf{0.47}$--$\mathbf{0.71}$ versus $0.08$--$0.19$. These consistent differences across models and benchmarks indicate stronger clustering in the representation space, supporting the relevance of the multiple-group structure predicted by our theory to frontier sparse-attention LLMs.

%% file: 4_conclusions.tex
\section{Conclusions}

We introduce \emph{opinion leader dynamics} to characterize two mechanisms for token clustering under sparse interactions: the \textcolor{explicitcolor}{\textbf{explicit}} model where attraction toward representative-induced maxima and the \textcolor{implicitcolor}{\textbf{implicit}} model consensus within disconnected interaction groups. Both admit reverse Wasserstein gradient flow formulations and converge exponentially to distinct limiting directions under suitable conditions. Finite-particle simulations illustrate these results. Across four benchmarks, \textcolor{explicitcolor}{\textbf{Kimi-K3}}, \textcolor{implicitcolor}{\textbf{MiniMax-M3}}, and \textcolor{explicitcolor}{\textbf{DeepSeek-}}\textcolor{implicitcolor}{\textbf{V4-Flash}} exhibit stronger clustering behavior than dense \textcolor{standardcolor}{\textbf{GLM-4.7-Flash}}, consistent with the predicted group structure. Together, these findings connect tractable convergence guarantees with representation patterns in frontier sparse-attention LLMs.


\section*{AI Use Statement}

In this work, we used generative AI tools to assist with code generation and to improve the clarity, grammar, and readability of the manuscript. Generative AI tools did not directly generate or alter the experiment results reported in this paper. The authors reviewed and tested all AI-assisted code and carefully verified all AI-assisted manuscript edits. We take full responsibility for the final content, including all text, claims, results, and associated artifacts.


%% file: Appendix.tex
\clearpage
\appendix
\appendixtableofcontents
\clearpage

\section{Discussion and Future Work}
\label{app:discussion}

Our analysis isolates two mechanisms through which restricted interactions support alignment within token groups while preserving separation between them. To obtain tractable convergence guarantees, we study fixed representatives in the explicit model and disconnected interaction components in the implicit model, together with suitable geometric conditions on initialization. These idealizations allow us to characterize distinct limiting directions and quantify convergence rates. 

An important future extension is to incorporate \emph{causal masking} into our analysis. Our current framework does not impose causal constraints, whereas all frontier LLMs examined in our experiments use causal attention. Causal masking makes token interactions directional and can break the symmetry underlying our gradient-flow formulation. We therefore plan to develop a framework for causal opinion leader dynamics, investigating how directed interactions affect the emergence, separation, and stability of token groups.

Despite these differences between theory and practice, the frontier sparse-attention LLMs studied here exhibit the qualitative behavior captured by our framework: their final-layer representations show distinct token groups and consistently stronger cluster separation than the dense-attention comparison model across four benchmarks. These observations support the relevance of the proposed mechanisms beyond the idealized setting, although they do not establish that trained models satisfy every theoretical assumption. Future work will examine which conditions are essential for this correspondence and extend the analysis to evolving representatives, time-dependent interaction patterns, and weak communication between groups.


\section{Related Work}
\label{app:A}

\noindent\textbf{Attention Dynamics.} A growing line of theoretical work models self-attention as an interacting particle system and studies it through mean-field analysis~\citep{geshkovski2023emergence,chen2025quantitative,geshkovski2025mathematical,rigollet2026mean}. Within this framework, \citet{bruno2025emergence} show that metastable clusters form and persist near uniform initialization as the number of tokens grows, while \citet{bruno2025multiscale} identify several regimes at moderate interaction strength, including one in which clusters merge slowly and sequentially. Complementary work considers restricted interaction patterns. \citet{karagodin2024clustering} analyze finite-token dynamics under causal masking, while \citet{liu2026krause} introduce a bounded-confidence sparse attention and study its clustering behavior.

\noindent\textbf{Sparse Architectures.} Sparse attention reduces computation by allowing each query to interact with only a subset of keys. Existing methods use local windows or predefined sparse patterns~\citep{child2019generating,beltagy2020longformer,zaheer2020big}, content-dependent hashing or routing~\citep{kitaev2020reformer,roy2021efficient,lu2025moba}, or representative tokens to guide retrieval~\citep{mohtashami2023landmark,hu2025hardware,leng2026understanding}. Native Sparse Attention (NSA) combines token compression, blockwise selection, and sliding-window attention~\citep{yuan2025native}. DeepSeek Sparse Attention (DSA) instead uses a lightweight learned indexer to select a query-dependent top-$k$ subset of preceding tokens~\citep{liu2025deepseek}. Recent theoretical work also shows that combining local and global attention can be more expressive than using either alone~\citep{li2026characterizing}. Motivated by these architectures, we formulate two idealized models within \emph{opinion leader dynamics} framework: one uses fixed representatives, while the other restricts interactions through a sparse graph.

\section{Extended Notation and Definitions}
\label{app:notation}

This section collects supplementary notation for the two models of \emph{opinion leader dynamics}. Table~\ref{tab:notation_explicit} records the notation associated with \emph{explicit opinion leader dynamics}, while Table~\ref{tab:notation_implicit} reports the notation used in \emph{implicit opinion leader dynamics}. The listed notation is used throughout the model formulations, assumptions, convergence results, and proofs developed in the main text and appendix.

\begin{table}[htbp]
  \caption{Extended key notation for \emph{explicit opinion leader dynamics}.}
  \label{tab:notation_explicit}
  \centering
  \small
  \renewcommand{\arraystretch}{1.35}
  \begin{tabular}{
    c
    @{\hspace{3em}}
    c
  }
    \toprule
    \textbf{Notation} & \textbf{Definition} \\
    \midrule

    $\operatorname{Crit}(\Phi_\nu)$
    &
    $\left\{
      x\in\mathbb S^{d-1}:
      \nabla^\circ\Phi_\nu(x)=0
    \right\}$
    \\

    $\mu_t^a$
    &
    $(\phi_t)_\#\mu_0^a,
    \qquad a=1,\ldots,m$
    \\

    $\theta_{ab}$
    &
    $d_{\mathbb S}(\ell_a,\ell_b),
    \qquad a\neq b$
    \\

    $\gamma_{ab}(\alpha)$
    &
    $\cos\bigl(\theta_{ab}-\alpha\bigr),
    \qquad a\neq b$
    \\

    $\mathbf I_a(\alpha)$
    &
    $\displaystyle
      \begin{aligned}
        &\rho_a r_a
        \exp\bigl(\beta r_a\cos\alpha\bigr)
        \sin^2\alpha\\
        &\quad
        -\sin\alpha
        \sum_{b\neq a}
        \rho_b r_b
        \exp\bigl(\beta r_b\gamma_{ab}(\alpha)\bigr)
        \sin\bigl(\theta_{ab}-\alpha\bigr),
        \qquad a=1,\ldots,m
      \end{aligned}
    $
    \\

    $\mathbf H_a(\alpha)$
    &
    $\displaystyle
      \begin{aligned}
        &\rho_a r_a
        \exp\bigl(\beta r_a\cos\alpha\bigr)
        \bigl(
          \cos\alpha-\beta r_a\sin^2\alpha
        \bigr)\\
        &\quad
        -\sum_{b\neq a}
        \rho_b r_b
        \exp\bigl(\beta r_b\gamma_{ab}(\alpha)\bigr)
        \bigl(1+\beta r_b\bigr),
        \qquad a=1,\ldots,m
      \end{aligned}
    $
    \\

    \bottomrule
  \end{tabular}
\end{table}


\begin{table}[htbp]
  \caption{Extended key notation for \emph{implicit opinion leader dynamics}.}
  \label{tab:notation_implicit}
  \centering
  \small
  \renewcommand{\arraystretch}{1.35}
  \begin{tabular}{
    c
    @{\hspace{4em}}
    c
  }
    \toprule
    \textbf{Notation} & \textbf{Definition} \\
    \midrule

    $\mathbf A_k$
    &
    $\left(A_{ij}^{(R)}\right)_{i,j\in\mathcal I_k}$
    \\

    $\mathbf D_k$
    &
    $\operatorname{diag}\!\left(
      \left(d_i^{(R)}\right)_{i\in\mathcal I_k}
    \right)$
    \\

    $\lambda_k$
    &
    $\lambda_2\!\left(
      \mathbf I_{n_k}
      -\mathbf D_k^{-1/2}
       \mathbf A_k
       \mathbf D_k^{-1/2}
    \right)$
    \\

    $\operatorname{vol}_k$
    &
    $\displaystyle\sum_{i\in\mathcal I_k} d_i^{(R)}$
    \\

    $\mathsf P_{ij}^k$
    &
    $\frac{A_{ij}^{(R)}}{d_i^{(R)}}$
    \\

    $\pi_i^k$
    &
    $\frac{d_i^{(R)}}{\operatorname{vol}_k}$
    \\

    $V_0^k$
    &
    $\frac{1}{2}
    \iint_{\mathbb S^{d-1}\times\mathbb S^{d-1}}
    \|x-y\|_2^2
    \,d\widehat{\mu}_0^k(x)
    \,d\widehat{\mu}_0^k(y)$
    \\

    $\chi_k$
    &
    $\frac{\operatorname{vol}_k}{
    n_k\min_{i\in\mathcal I_k}
    d_i^{(R)}}$
    \\

    $\kappa_k$
    &
    $\lambda_k
    \cos(2\alpha)
    e^{\beta\cos(2\alpha)}$
    \\

    $C_k$
    &
    $\frac{2e^{\beta(1-\cos(2\alpha))}}{
        \lambda_k\cos(2\alpha)}$
    \\

    $\kappa_{\min}$
    &
    $\min_{1\leq k\leq K}
    \kappa_k$
    \\

    $C_G$
    &
    $\left(
        \sum_{k=1}^{K}
        m_k\chi_kC_k^2V_0^k
    \right)^{1/2}$
    \\
    
    \bottomrule
  \end{tabular}
\end{table}

\section{Interpreting Sparse Attention through Opinion Leader Dynamics}
\label{app:connections_sparse_attention}

A wide range of existing sparse attention architectures can be interpreted through the \textcolor{explicitcolor}{\textbf{Explicit}} and \textcolor{implicitcolor}{\textbf{Implicit}} models of our \emph{opinion leader dynamics}. As summarized in Table~\ref{tab:sparse_attention_connections}, Landmark Attention~\citep{mohtashami2023landmark}, Kimi Delta Attention~\citep{team2026kimi}, and Hierarchical Sparse Attention~\citep{hu2025hardware} are \textcolor{explicitcolor}{\textbf{Explicit}} Sparse Attention for their \textbf{representative-based mechanisms}; Longformer~\citep{beltagy2020longformer}, BigBird~\citep{zaheer2020big}, DeepSeek Sparse Attention (V3.2)~\citep{liu2025deepseek}, and MiniMax Sparse Attention~\citep{lai2026minimax} are \textcolor{implicitcolor}{\textbf{Implicit}} Sparse Attention for their \textbf{restricted interaction mechanisms}; NSA~\citep{yuan2025native} combines three branches: compressed block representations, top-$k$ block selection, and local sliding-window attention, which together incorporate both \textbf{explicit representatives} and a \textbf{restricted interaction pattern}. DeepSeek-V4 Attention similarly combines the two mechanisms through its Compressed Sparse Attention (CSA) and Heavily Compressed Attention (HCA) modules. Both compress token groups into key-value entries that serve as \textbf{explicit representatives}. Their local sliding-window branches, together with query-dependent selection in CSA, define a \textbf{restricted interaction pattern}~\citep{xu2026deepseek}. Accordingly, we classify NSA and DeepSeek-V4 Attention as \textcolor{explicitcolor}{\textbf{Hyb}}\textcolor{implicitcolor}{\textbf{rid}} Sparse Attention.

\begin{table}[htbp]
    \centering
    \caption{Connections between existing sparse attention mechanisms and our \emph{opinion leader dynamics} framework. For each method, we summarize its main design and identify its corresponding primary dynamical interpretation within our framework: \textcolor{explicitcolor}{\textbf{Explicit}}, \textcolor{implicitcolor}{\textbf{Implicit}}, and \textcolor{explicitcolor}{\textbf{Hyb}}\textcolor{implicitcolor}{\textbf{rid}}.}
    \label{tab:sparse_attention_connections}
    \small
    \setlength{\tabcolsep}{5pt}
    \renewcommand{\arraystretch}{1.15}
    \renewcommand{\tabularxcolumn}[1]{m{#1}}
    
    \begin{tabularx}{\linewidth}{@{}
     >{\centering\arraybackslash}m{0.24\linewidth}
     >{\centering\arraybackslash}X
     >{\centering\arraybackslash}m{0.13\linewidth}
    @{}}
    \toprule
    \textbf{Attention mechanism}
    & \textbf{Description}
    & \textbf{Interpretation} \\
    \midrule

    \textbf{Landmark Attention}\newline
    \citep{mohtashami2023landmark}
    & Partitions the sequence into fixed-size blocks, each followed by a landmark token that serves as an \textbf{explicit representative}. Query-landmark attention scores gate access to the associated blocks during training and guide their retrieval at inference.
    & \textcolor{explicitcolor}{\textbf{Explicit}} \\
    \addlinespace

    \textbf{Kimi Delta Attention}\newline
    \citep{team2026kimi}
    & Compresses past-token information into a gated recurrent state that serves as an \textbf{explicit representative} and is updated at each step to provide query-dependent context.
    & \textcolor{explicitcolor}{\textbf{Explicit}} \\
    \addlinespace

    \textbf{Hierarchical Sparse Attention}\newline
    \citep{hu2025hardware,leng2026understanding}
    & Encodes each fixed-length chunk as a retrieval representation that can be interpreted as an \textbf{explicit representative}. Each query scores these representations, attends within the top-$k$ chunks, and combines their outputs with chunk-level weights.
    & \textcolor{explicitcolor}{\textbf{Explicit}} \\
    \addlinespace
   
    \textbf{Longformer}\newline
    \citep{beltagy2020longformer}
    & Defines a \textbf{sparse interaction pattern} that combines local sliding-window attention with full-sequence interactions through designated global tokens.
    & \textcolor{implicitcolor}{\textbf{Implicit}} \\
    \addlinespace

    \textbf{BigBird}\newline
    \citep{zaheer2020big}
    & Defines a \textbf{sparse interaction pattern} that combines local windows, random connections, and global tokens to capture long-range dependencies.
    & \textcolor{implicitcolor}{\textbf{Implicit}} \\
    \addlinespace

    \textbf{DeepSeek Sparse Attention (V3.2)}\newline
    \citep{liu2025deepseek}
    & Restricts interactions for each query to the top-$k$ preceding tokens selected by a lightweight learned indexer, yielding a query-dependent \textbf{sparse interaction pattern}.
    & \textcolor{implicitcolor}{\textbf{Implicit}} \\
    \addlinespace

    \textbf{MiniMax Sparse Attention}\newline
    \citep{lai2026minimax}
    & Restricts attention within each query-head group to the top-$k$ key-value blocks while always retaining the local block, thereby defining a blockwise \textbf{sparse interaction pattern}.
    & \textcolor{implicitcolor}{\textbf{Implicit}} \\
    \addlinespace

    \textbf{Native Sparse Attention}\newline
    \citep{yuan2025native}
    & Combines compressed blocks that act as \textbf{explicit representatives} with a \textbf{sparse interaction pattern} defined by top-$k$ block selection and local sliding-window attention.
    & \textcolor{explicitcolor}{\textbf{Hyb}}\textcolor{implicitcolor}{\textbf{rid}} \\
    \addlinespace

    \textbf{DeepSeek-V4 Attention}\newline
    \citep{xu2026deepseek}
    & Combines Compressed Sparse Attention (CSA) and Heavily Compressed Attention (HCA). Compressed key-value entries in both modules serve as \textbf{explicit representatives}, while their sliding-window branches and query-dependent selection in CSA define a \textbf{sparse interaction pattern}.
    & \textcolor{explicitcolor}{\textbf{Hyb}}\textcolor{implicitcolor}{\textbf{rid}} \\
    \bottomrule
  \end{tabularx}
\end{table}

\section{Standard Attention Dynamics}
\label{app:standard_attention}

Prior work has modeled Standard Attention as an interacting particle system on the unit sphere~\citep{chen2025quantitative,geshkovski2025mathematical}. We briefly review their derivations here for completeness. The analysis of our framework instead uses fixed-potential gradient analysis for \emph{explicit opinion leader dynamics} and graph decomposition with spectral-gap estimates for \emph{implicit opinion leader dynamics}.

Let $x_i(t)\in\mathbb S^{d-1}$ denote the representation of token $i$ at time $t$. For $\beta>0$, the tokens evolve according to
\begin{equation}
    \dot{x}_i(t)
    =
    \mathbf{P}_{x_i(t)}^{\perp}
    \left(
        \frac{1}{n}\sum_{j=1}^{n}
        e^{\beta\langle x_i(t),x_j(t)\rangle}x_j(t)
    \right).
    \label{eq:sa_particle_dynamics}
\end{equation}

The associated empirical measure $\mu_t=\frac{1}{n}\sum_{i=1}^{n}\delta_{x_i(t)}$ satisfies
\begin{equation}
    \partial_t\mu_t
    +
    \operatorname{div}^{\circ}
    \left(
        \mu_t\mathcal{X}_{\mu_t,\beta}
    \right)
    =0,
    \qquad
    \mathcal{X}_{\mu_t,\beta}(x)
    =
    \int_{\mathbb{S}^{d-1}}
    \mathbf{P}_{x}^{\perp}(y)
    e^{\beta\langle x,y\rangle}
    \,d\mu_t(y),
    \label{eq:sa_mean_field_dynamics}
\end{equation}
where $\operatorname{div}^{\circ}$ denotes the spherical divergence. More generally, Eq.~\ref{eq:sa_mean_field_dynamics} defines the mean-field dynamics for any probability measure on $\mathbb S^{d-1}$, with the associated interaction energy
\begin{equation}
    \mathbf{E}_{\beta}[\mu]
    =
    \frac{1}{2\beta}
    \iint_{\mathbb S^{d-1}\times\mathbb S^{d-1}}
    e^{\beta\langle x,y\rangle}
    \,d\mu(x)d\mu(y).
    \label{eq:sa_energy_main}
\end{equation}

Let $\chi$ be a signed measure with zero total mass, and suppose that $\mu_\varepsilon=\mu+\varepsilon\chi$ remains a probability measure for sufficiently small $|\varepsilon|$. Differentiating Eq.~\ref{eq:sa_energy_main} at $\varepsilon=0$ and using the symmetry of the kernel gives
\begin{equation}
    \left.
    \frac{d}{d\varepsilon}
    \mathbf E_\beta[\mu_\varepsilon]
    \right|_{\varepsilon=0}
    =
    \int_{\mathbb S^{d-1}}
    \left(
        \frac{1}{\beta}
        \int_{\mathbb S^{d-1}}
        e^{\beta\langle x,y\rangle}
        d\mu(y)
    \right)
    d\chi(x).
\end{equation}
Therefore, the first variation of $\mathbf E_\beta$ is
\begin{equation}
    \frac{\delta\mathbf E_\beta}{\delta\mu}(x)
    =
    \frac{1}{\beta}
    \int_{\mathbb S^{d-1}}
    e^{\beta\langle x,y\rangle}
    d\mu(y).
    \label{eq:app_sa_first_variation}
\end{equation}
Taking the spherical gradient yields
\begin{align}
    \nabla_x^\circ
    \frac{\delta\mathbf E_\beta}{\delta\mu}(x)
    &=
    \int_{\mathbb S^{d-1}}
    \mathbf P_x^\perp(y)
    e^{\beta\langle x,y\rangle}
    d\mu(y)
    \notag\\
    &=
    \mathcal X_{\mu,\beta}(x).
    \label{eq:app_sa_gradient_velocity}
\end{align}
Hence, Eq.~\ref{eq:sa_mean_field_dynamics} defines a \emph{reverse Wasserstein gradient flow} of $\mathbf E_\beta$:
\begin{equation}
    \partial_t\mu_t
    +
    \operatorname{div}^{\circ}
    \left(
        \mu_t
        \nabla^\circ
        \frac{\delta\mathbf E_\beta}{\delta\mu_t}
    \right)
    =
    0.
    \label{eq:app_sa_gradient_ascent}
\end{equation}

Define the squared Wasserstein slope by
\begin{equation}
    I_\beta(\mu)
    =
    \int_{\mathbb S^{d-1}}
    \|\mathcal X_{\mu,\beta}(x)\|_2^2
    d\mu(x).
    \label{eq:app_sa_slope}
\end{equation}
Using Eq.~\ref{eq:app_sa_gradient_ascent} and integration by parts on the sphere, we obtain
\begin{align}
    \frac{d}{dt}\mathbf E_\beta[\mu_t]
    &=
    \int_{\mathbb S^{d-1}}
    \left\langle
        \nabla_x^\circ
        \frac{\delta\mathbf E_\beta}{\delta\mu_t}(x),
        \mathcal X_{\mu_t,\beta}(x)
    \right\rangle
    d\mu_t(x)
    \notag\\
    &=
    \int_{\mathbb S^{d-1}}
    \|\mathcal X_{\mu_t,\beta}(x)\|_2^2
    d\mu_t(x)
    \notag\\
    &=
    I_\beta(\mu_t)
    \geq
    0.
    \label{eq:app_sa_energy_production}
\end{align}
Thus, $\mathbf E_\beta[\mu_t]$ is nondecreasing along the flow.

For any $x,y\in\mathbb S^{d-1}$, we have $\langle x,y\rangle\leq1$. Therefore, every probability measure $\mu\in\mathcal P(\mathbb S^{d-1})$ satisfies
\begin{align}
    \mathbf E_\beta[\mu]
    &=
    \frac{1}{2\beta}
    \iint_{\mathbb S^{d-1}\times\mathbb S^{d-1}}
    e^{\beta\langle x,y\rangle}
    d\mu(x)d\mu(y)
    \notag\\
    &\leq
    \frac{1}{2\beta}
    \iint_{\mathbb S^{d-1}\times\mathbb S^{d-1}}
    e^\beta
    d\mu(x)d\mu(y)
    \notag\\
    &=
    \frac{e^\beta}{2\beta}
    =
    \mathbf E_\beta[\delta_{x_0}],
    \qquad
    x_0\in\mathbb S^{d-1}.
    \label{eq:app_sa_energy_upper_bound}
\end{align}
Equality holds only if $\langle x,y\rangle=1$ for $\mu\otimes\mu$-almost every $(x,y)$. Since $x$ and $y$ are unit vectors, this condition implies that $\mu$ is a Dirac measure. Hence,
\begin{equation}
    \mathbf E_\beta^\star
    =
    \sup_{\mu\in\mathcal P(\mathbb S^{d-1})}
    \mathbf E_\beta[\mu]
    =
    \frac{e^\beta}{2\beta},
    \label{eq:app_sa_energy_maximum}
\end{equation}
and the global maximizers of $\mathbf E_\beta$ are precisely the Dirac measures.

We next recall two convergence results established by \citet[\emph{Theorems~2.3 and~2.5}]{chen2025quantitative}. 

\begin{theorem}[\textbf{Localized initialization}]
\label{thm:sa_localized_consensus}
Let $d\geq 2$, $\beta>0$, $u\in\mathbb{S}^{d-1}$, and $\alpha\in[0,\pi/2)$. We define $S_{\alpha}^{+}(u)=\left\{x \in \mathbb{S}^{d-1}:\langle x,u\rangle \geq \cos\alpha\right\}$, which means the closed spherical cap of radius $\alpha$ centered at $u$. Assume that
\begin{equation}
    \operatorname{supp}(\mu_0)
    \subseteq
    S_{\alpha}^{+}(u),
    \qquad
    10(1+\sqrt{\beta})\tan\alpha
    \leq
    1.
    \label{eq:sa_localized_condition_main}
\end{equation}
Then $S_{\alpha}^{+}(u)$ is forward invariant, and there exists $x_{\infty}\in S_{\alpha}^{+}(u)$ such that, for every $t\geq0$,
\begin{equation}
    W_2(\mu_t,\delta_{x_{\infty}})
    \leq
    20e^{-\beta}
    e^{-e^{\beta}t/20}
    I_{\beta}(\mu_0)^{1/2},
    \qquad
    I_{\beta}(\mu)
    =
    \int_{\mathbb{S}^{d-1}}
    \left\|
        \mathcal{X}_{\mu,\beta}(x)
    \right\|_2^2
    \,d\mu(x).
    \label{eq:sa_localized_rate_main}
\end{equation}
Moreover, any probability measure $\mu$ supported on
$S_{\alpha}^{+}(u)$ satisfies the Polyak--\L{}ojasiewicz inequality
\begin{equation}
    \mathbf{E}_{\beta}^{\star}
    -
    \mathbf{E}_{\beta}[\mu]
    \leq
    10e^{-\beta}I_{\beta}(\mu),
    \qquad
    \mathbf{E}_{\beta}^{\star}
    =
    \frac{e^{\beta}}{2\beta}.
    \label{eq:sa_pl_main}
\end{equation}
\end{theorem}

\begin{proof}
We follow the proof of \emph{Theorem~2.3} in \citet{chen2025quantitative}, written in our notation. The required localization conditions are
\begin{equation}
    \operatorname{supp}(\mu_0)
    \subseteq
    S_\alpha^+(u),
    \qquad
    10(1+\sqrt{\beta})\tan\alpha
    \leq
    1.
    \label{eq:app_sa_cap_condition}
\end{equation}

Let $\Phi_{0,t}$ denote the characteristic flow generated by $\mathcal X_{\mu_t,\beta}$, so that $\mu_t=(\Phi_{0,t})_\#\mu_0$. Define $m(t)=\min_{z\in\operatorname{supp}(\mu_0)}\langle\Phi_{0,t}(z),u\rangle$. At almost every $t$, the envelope formula gives
\begin{equation}
    m'(t)
    =
    \min_{x\in A_t}
    \langle\mathcal X_{\mu_t,\beta}(x),u\rangle,
    \qquad
    A_t
    =
    \left\{
        x\in\operatorname{supp}(\mu_t):
        \langle x,u\rangle=m(t)
    \right\}.
    \label{eq:app_sa_envelope_derivative}
\end{equation}
For $x\in A_t$ and $y\in\operatorname{supp}(\mu_t)$, the definition of $m(t)$ gives
\begin{equation}
    \langle y,u\rangle
    \geq
    m(t)
    =
    \langle x,u\rangle.
\end{equation}
Therefore, whenever $m(t)>0$,
\begin{align}
    \langle\mathcal X_{\mu_t,\beta}(x),u\rangle
    &=
    \int_{\mathbb S^{d-1}}
    \left(
        \langle y,u\rangle
        -
        \langle x,y\rangle\langle x,u\rangle
    \right)
    e^{\beta\langle x,y\rangle}
    d\mu_t(y)
    \notag\\
    &\geq
    m(t)
    \int_{\mathbb S^{d-1}}
    \left(
        1-\langle x,y\rangle
    \right)
    e^{\beta\langle x,y\rangle}
    d\mu_t(y)
    \notag\\
    &\geq
    0.
    \label{eq:app_sa_cap_barrier}
\end{align}
Since $m(0)\geq\cos\alpha>0$ and $m'(t)\geq0$ whenever $m(t)>0$, absolute continuity yields
\begin{equation}
    m(t)
    \geq
    m(0)
    \geq
    \cos\alpha,
    \qquad
    t\geq0.
\end{equation}
It follows that
\begin{equation}
    \operatorname{supp}(\mu_t)
    \subseteq
    S_\alpha^+(u),
    \qquad
    t\geq0.
    \label{eq:app_sa_cap_invariance}
\end{equation}
This forward-invariance argument only requires $\alpha<\pi/2$. The stronger condition in Eq.~\ref{eq:app_sa_cap_condition} is used below to obtain the quantitative slope decay.

Set $I_t=I_\beta(\mu_t)$. Under Eq.~\ref{eq:app_sa_cap_condition}, specializing \emph{Lemma~4.1 and Remark~4.2} of \citet{chen2025quantitative} to $A=\beta I_d$ and $\phi'(r)=e^r$ yields
\begin{equation}
    \frac{d}{dt}I_t
    \leq
    -\frac{e^\beta}{10}I_t.
    \label{eq:app_sa_slope_differential}
\end{equation}
Therefore,
\begin{equation}
    I_t
    \leq
    e^{-e^\beta t/10}I_0.
    \label{eq:app_sa_slope_decay}
\end{equation}

The metric derivative of the continuity equation satisfies $|\dot\mu_t|_{W_2}\leq I_t^{1/2}$. Hence, for $0\leq s\leq t$,
\begin{align}
    W_2(\mu_s,\mu_t)
    &\leq
    \int_s^t I_r^{1/2}\,dr
    \notag\\
    &\leq
    20e^{-\beta}
    e^{-e^\beta s/20}
    I_0^{1/2}.
    \label{eq:app_sa_cauchy_bound}
\end{align}
Thus, $\{\mu_t\}_{t\geq0}$ is a Cauchy curve in $W_2$ and converges to some $\mu_\infty\in\mathcal P(\mathbb S^{d-1})$. Since $S_\alpha^+(u)$ is closed and $\mu_t\bigl(S_\alpha^+(u)\bigr)=1$ for $t\geq0$, the Portmanteau theorem gives
\begin{equation}
    \operatorname{supp}(\mu_\infty)
    \subseteq
    S_\alpha^+(u).
    \label{eq:app_sa_limit_cap}
\end{equation}

Smoothness of the kernel on the compact sphere ensures that $I_\beta$ is continuous with respect to $W_2$. Since $\mu_t\to\mu_\infty$ in $W_2$ and $I_\beta(\mu_t)\to0$ by Eq.~\ref{eq:app_sa_slope_decay}, we obtain
\begin{equation}
    I_\beta(\mu_\infty)
    =
    \lim_{t\to\infty}I_\beta(\mu_t)
    =
    0,
\end{equation}
which shows that the field $\mathcal X_{\mu_\infty,\beta}$ vanishes $\mu_\infty$-almost everywhere. Continuity in $x$ extends this identity to the support:
\begin{equation}
    \mathcal X_{\mu_\infty,\beta}(x)
    =
    0,
    \qquad
    x\in\operatorname{supp}(\mu_\infty).
\label{eq:app_sa_stationary_limit}
\end{equation}

Choose $x_0\in\operatorname{supp}(\mu_\infty)$ minimizing $\langle x,u\rangle$. Taking the inner product of Eq.~\ref{eq:app_sa_stationary_limit} with $u$ gives
\begin{equation}
    0
    =
    \int_{\mathbb S^{d-1}}
    \left(
        \langle y,u\rangle
        -
        \langle x_0,y\rangle\langle x_0,u\rangle
    \right)
    e^{\beta\langle x_0,y\rangle}
    d\mu_\infty(y).
    \label{eq:app_sa_stationary_cap_identity}
\end{equation}
By the choice of $x_0$, the integrand is bounded below by
\begin{equation}
    \langle x_0,u\rangle
    \left(
        1-\langle x_0,y\rangle
    \right)
    e^{\beta\langle x_0,y\rangle}
    \geq
    0.
\end{equation}
Moreover, $\langle x_0,u\rangle\geq\cos\alpha>0$. Hence, Eq.~\ref{eq:app_sa_stationary_cap_identity} implies $\langle x_0,y\rangle=1$ for $\mu_\infty$-almost every $y$. Therefore, $\mu_\infty=\delta_{x_\infty}$ for some $x_\infty\in S_\alpha^+(u)$. Letting $t\to\infty$ in Eq.~\ref{eq:app_sa_cauchy_bound} yields
\begin{equation}
    W_2(\mu_t,\delta_{x_\infty})
    \leq
    20e^{-\beta}
    e^{-e^\beta t/20}
    I_\beta(\mu_0)^{1/2}.
\end{equation}
This proves Eq.~\ref{eq:sa_localized_rate_main}.

The preceding argument applies to any probability measure supported on $S_\alpha^+(u)$. Start the flow from such a measure $\mu$. By Eqs.~\ref{eq:app_sa_energy_production} and~\ref{eq:app_sa_slope_decay},
\begin{align}
    \mathbf E_\beta^\star
    -
    \mathbf E_\beta[\mu]
    &=
    \lim_{T\to\infty}
    \left(
        \mathbf E_\beta[\mu_T]
        -
        \mathbf E_\beta[\mu]
    \right)
    \notag\\
    &=
    \int_0^\infty
    I_\beta(\mu_t)\,dt
    \notag\\
    &\leq
    10e^{-\beta}
    I_\beta(\mu).
    \label{eq:E_beta_inequality}
\end{align}
Because $\mathbf E_\beta^\star=\frac{e^\beta}{2\beta}$, Eq.~\ref{eq:E_beta_inequality} proves Eq.~\ref{eq:sa_pl_main} and completes the proof of Theorem~\ref{thm:sa_localized_consensus}.
\end{proof}

Theorem~\ref{thm:sa_localized_consensus} establishes exponential convergence to a single consensus direction, with rate $e^\beta/20$ uniformly in the number of tokens. Theorem~\ref{thm:sa_general_consensus} moves beyond this setting by replacing the localization condition with $L^2$ regularity and a nonzero initial mean.

\begin{theorem}[\textbf{$L^2$ mean-field initialization}]
\label{thm:sa_general_consensus}
Let $d\geq 2$, and let $\sigma$ denote the normalized uniform measure on $\mathbb{S}^{d-1}$. Assume that
\begin{equation}
    d\mu_0
    =
    f_0\,d\sigma,
    \qquad
    f_0
    \in
    L^2(\mathbb{S}^{d-1},\sigma),
    \qquad
    R_0
    =
    \left\|
        \int_{\mathbb{S}^{d-1}}
        x\,d\mu_0(x)
    \right\|_2
    >
    0.
    \label{eq:sa_general_assumptions_main}
\end{equation}
Then, for every $t>0$, the measure $\mu_t$ admits a density $f_t\in L^2(\mathbb{S}^{d-1},\sigma)$ with respect to $\sigma$. Moreover, there exist constants $\beta_0,C_0,T_0>0$, depending on $\mu_0$, such that, for every $0<\beta<\beta_0$, there exists $x_\infty\in\mathbb{S}^{d-1}$ satisfying
\begin{equation}
    W_2(\mu_t,\delta_{x_\infty})
    \leq
    C_0e^{-t/100},
    \qquad
    t>T_0.
    \label{eq:sa_general_rate_main}
\end{equation}
\end{theorem}

\begin{proof}
We specialize the quantitative estimates of~\citet{chen2025quantitative} to Standard Attention dynamics. Take
\begin{equation}
    A
    =
    I_d,
    \qquad
    \varphi'(r)
    =
    e^{\beta r},
    \qquad
    r\in[-1,1].
    \label{eq:app_sa_general_specialization}
\end{equation}
The parametrization in Eq.~\ref{eq:app_sa_general_specialization} yields the same kernel $e^{\beta\langle x,y\rangle}$ as the localized choice $A=\beta I_d$ and $\varphi'(r)=e^r$, while making the dependence on the small parameter $\beta$ explicit.

For this choice, the perturbation parameter in \citet{chen2025quantitative} satisfies
\begin{align}
    \varepsilon_\beta
    &=
    3
    \|
        e^{\beta(\cdot)}-1
    \|_{C^1([-1,1])}
    \notag\\
    &\leq
    3
    \left(
        e^\beta-1+\beta e^\beta
    \right)
    \longrightarrow
    0
    \qquad
    \text{as }
    \beta\to0.
    \label{eq:app_sa_general_perturbation}
\end{align}
Therefore, there exists $\beta_0=\beta_0(\mu_0)>0$ such that, for every $0<\beta<\beta_0$,
\begin{equation}
    \varepsilon_\beta
    \leq
    \min
    \left\{
        \frac{1}{100},
        c_{\mathrm u}R_0^6
    \right\}.
    \label{eq:app_sa_beta_smallness}
\end{equation}
These are the small-perturbation conditions required by
\emph{Theorems~3.5 and~3.8} of~\citet{chen2025quantitative}.


Since the velocity field is smooth on the compact sphere, its characteristic flow is a diffeomorphism for every finite $t$. Hence, absolute continuity and $L^2$ regularity are preserved: if $d\mu_0=f_0\,d\sigma$ with $f_0\in L^2(\mathbb S^{d-1},\sigma)$, then
\begin{equation}
    d\mu_t
    =
    f_t\,d\sigma,
    \qquad
    f_t
    \in
    L^2(\mathbb S^{d-1},\sigma),
    \qquad
    t>0.
\end{equation}


Define
\begin{equation}
    M_t
    =
    \int_{\mathbb S^{d-1}}
    x\,d\mu_t(x),
    \qquad
    R_t
    =
    \|M_t\|_2,
    \qquad
    U_t
    =
    \frac{M_t}{R_t}.
    \label{eq:app_sa_general_mean}
\end{equation}
After possibly reducing $\beta_0$, \emph{Lemma~6.6} of \citet{chen2025quantitative} yields a constant $\lambda\in(0,1)$ such that
\begin{equation}
    R_t
    \geq
    \lambda R_0
    >
    0,
    \qquad
    t\geq0.
\end{equation}
Thus, $U_t$ is well defined for all $t\geq0$.

Fix $\alpha=\pi/100$ and define the mass outside the moving cap by
\begin{equation}
    p_t
    =
    \mu_t
    \left(
        \mathbb S^{d-1}
        \setminus
        S_\alpha^+(U_t)
    \right).
    \label{eq:app_sa_outside_mass}
\end{equation}
\emph{Theorem~3.5} of~\citet{chen2025quantitative} gives
\begin{equation}
    \frac{d}{dt}I_\beta(\mu_t)
    \leq
    -I_\beta(\mu_t)
    +
    100p_t.
    \label{eq:app_sa_general_entropy}
\end{equation}
Moreover, \emph{Theorem~3.8} of~\citet{chen2025quantitative} gives a time $T_{\mathrm{out}}>0$ such that
\begin{equation}
    p_t
    \leq
    \|f_0\|_{L^2(\mathbb S^{d-1},\sigma)}
    \exp\left(
        -\frac{d-1}{16}(t-T_{\mathrm{out}})
    \right),
    \qquad
    t\geq T_{\mathrm{out}}.
    \label{eq:app_sa_general_concentration}
\end{equation}

Set $b=\frac{d-1}{16}$. Combining Eqs.~\ref{eq:app_sa_general_entropy} and~\ref{eq:app_sa_general_concentration} and applying Gr\"onwall's inequality gives
\begin{align}
    I_\beta(\mu_t)
    &\leq
    e^{-(t-T_{\mathrm{out}})}
    I_\beta(\mu_{T_{\mathrm{out}}})
    \notag\\
    &\quad+
    100
    \|f_0\|_{L^2(\mathbb S^{d-1},\sigma)}
    \int_{T_{\mathrm{out}}}^{t}
    e^{-(t-r)}
    e^{-b(r-T_{\mathrm{out}})}
    \,dr,
    \qquad
    t\geq T_{\mathrm{out}}.
    \label{eq:app_sa_general_gronwall}
\end{align}
Since $d\ge 2$, we have $b=(d-1)/16\ge 1/16$. We may therefore choose $a=1/50<\min\{1,b\}$. For $t\ge T_{\mathrm{out}}$, the convolution term satisfies
\begin{equation}
    \int_{T_{\mathrm{out}}}^{t}
    e^{-(t-r)}e^{-b(r-T_{\mathrm{out}})}\,dr
    \le
    C_{a,b}e^{-a(t-T_{\mathrm{out}})}.
\end{equation}
Combining this estimate with the preceding bound and absorbing $T_{\mathrm{out}}$ into the constant, there exists $C_I>0$ such that
\begin{equation}
    I_\beta(\mu_t)
    \le
    C_I e^{-t/50},
    \qquad
    t\ge T_{\mathrm{out}}.
    \label{eq:app_sa_general_slope_decay}
\end{equation}


Integrating the metric derivative, for all $t\geq s\geq T_{\mathrm{out}}$, gives
\begin{align}
    W_2(\mu_s,\mu_t)
    &\leq
    \int_s^t
    I_\beta(\mu_r)^{1/2}\,dr
    \notag\\
    &\leq
    C_I^{1/2}
    \int_s^t e^{-r/100}\,dr
    \notag\\
    &\leq
    100C_I^{1/2}e^{-s/100}.
    \label{eq:app_sa_general_cauchy}
\end{align}
Therefore, $\{\mu_t\}_{t\geq0}$ is a Cauchy curve in $W_2$ and converges to some $\mu_\infty\in\mathcal P(\mathbb S^{d-1})$. Letting $t\to\infty$ in Eq.~\ref{eq:app_sa_general_cauchy} yields
\begin{equation}
    W_2(\mu_s,\mu_\infty)
    \leq
    100C_I^{1/2}e^{-s/100},
    \qquad
    s\geq T_{\mathrm{out}}.
    \label{eq:app_sa_general_preliminary_rate}
\end{equation}

Moreover, the continuity of $I_\beta$ under $W_2$ convergence and Eq.~\ref{eq:app_sa_general_slope_decay} imply
\begin{equation}
    I_\beta(\mu_\infty)
    =
    0,
\end{equation}
which means $\mu_\infty$ is stationary.

It remains to identify the stationary limit. By compactness of $\mathbb S^{d-1}$, there exists a sequence $t_k\to\infty$ and a direction $U_\infty\in\mathbb S^{d-1}$ such that
\begin{equation}
    U_{t_k}
    \longrightarrow
    U_\infty.
\end{equation}
Define $g(x)=\left(\cos\alpha-\langle x,U_\infty\rangle\right)_+$. For every $k$, Eq.~\ref{eq:app_sa_outside_mass} gives
\begin{equation}
    \int_{\mathbb S^{d-1}}
    g(x)\,d\mu_{t_k}(x)
    \leq
    \|U_{t_k}-U_\infty\|_2
    +
    2p_{t_k}.
\end{equation}
By Eq.~\ref{eq:app_sa_general_concentration}, the right-hand side converges to zero. Passing to the limit yields
\begin{equation}
    \operatorname{supp}(\mu_\infty)
    \subseteq
    S_\alpha^+(U_\infty).
    \label{eq:app_sa_general_limit_cap}
\end{equation}
Since $\mu_\infty$ is stationary, the minimizer argument used in Eq.~\ref{eq:app_sa_stationary_cap_identity} applies with $u=U_\infty$. Therefore,
\begin{equation}
    \mu_\infty
    =
    \delta_{x_\infty}
\end{equation}
for some $x_\infty\in S_\alpha^+(U_\infty)$.

Combining Eq.~\ref{eq:app_sa_general_preliminary_rate} with $\mu_\infty=\delta_{x_\infty}$, and setting $C_0=100C_I^{1/2}$, $T_0=T_{\mathrm{out}}$, we obtain
\begin{equation}
    W_2(\mu_t,\delta_{x_\infty})
    \leq
    C_0e^{-t/100},
    \qquad
    t\geq T_0.
\end{equation}
This proves Eq.~\ref{eq:sa_general_rate_main} and completes the proof of Theorem~\ref{thm:sa_general_consensus}.
\end{proof}

Theorem~\ref{thm:sa_general_consensus} guarantees exponential convergence after $T_0$ with decay $e^{-t/100}$. The restriction to small $\beta$ is essential, as large $\beta$ can lead to non-consensus limits~\citep{chen2025quantitative}.

\section{Supplementary Proofs}
\label{app:supplementary}

\subsection{Explicit Opinion Leader Dynamics}
\label{app:fixed_leader_proof}

\subsubsection{Reverse Wasserstein Gradient Flow}
\label{app:explicit_gradient_flow}

Let $\chi$ be a signed measure with zero total mass and set $\mu_\varepsilon=\mu+\varepsilon\chi.$ Since $\nu$ is fixed, differentiating Eq.~\ref{eq:fixed_leader_energy} at $\varepsilon=0$ gives
\begin{equation}
    \left.
    \frac{d}{d\varepsilon}
    \mathbf E_\beta[\mu_\varepsilon|\nu]
    \right|_{\varepsilon=0}
    =
    \int_{\mathbb S^{d-1}}
    \left(
        \frac{1}{\beta}
        \int_{\mathbb R^d}
        e^{\beta\langle x,z\rangle}
        d\nu(z)
    \right)
    d\chi(x).
    \label{eq:app_fixed_energy_variation}
\end{equation}
Hence,
\begin{equation}
    \frac{\delta\mathbf E_\beta[\mu|\nu]}
    {\delta\mu}(x)
    =
    \frac{1}{\beta}
    \int_{\mathbb R^d}
    e^{\beta\langle x,z\rangle}
    d\nu(z)
    =
    \Phi_\nu(x).
    \label{eq:ola_first_variation}
\end{equation}
Taking the spherical gradient yields
\begin{align}
    \nabla_x^\circ
    \frac{\delta\mathbf E_\beta[\mu|\nu]}
    {\delta\mu}(x)
    &=
    \int_{\mathbb R^d}
    \mathbf P_x^\perp(z)
    e^{\beta\langle x,z\rangle}
    d\nu(z)
    \notag\\
    &=
    \mathcal X_{\nu,\beta}(x).
    \label{eq:app_fixed_wasserstein_gradient}
\end{align}

Therefore, Eq.~\ref{eq:fixed_leader_transport} is a \emph{reverse Wasserstein gradient flow} of $\mathbf E_\beta[\cdot|\nu]$:
\begin{equation}
    \partial_t\mu_t
    +
    \operatorname{div}^\circ
    \left(
        \mu_t
        \nabla_x^\circ
        \frac{\delta\mathbf E_\beta[\mu_t|\nu]}
        {\delta\mu_t}
    \right)
    =
    0.
    \label{eq:app_fixed_wasserstein_ascent}
\end{equation}

Define the corresponding squared Wasserstein slope by
\begin{equation}
    \mathcal I_\beta[\mu|\nu]
    =
    \int_{\mathbb S^{d-1}}
    \lVert\mathcal X_{\nu,\beta}(x)\rVert_2^2
    d\mu(x).
    \label{eq:app_fixed_wasserstein_slope}
\end{equation}
Using Eq.~\ref{eq:app_fixed_wasserstein_ascent} and integration by parts on the sphere, we obtain
\begin{align}
    \frac{d}{dt}
    \mathbf E_\beta[\mu_t|\nu]
    &=
    \int_{\mathbb S^{d-1}}
    \left\langle
        \nabla_x^\circ
        \frac{\delta\mathbf E_\beta[\mu_t|\nu]}
        {\delta\mu_t}(x),
        \mathcal X_{\nu,\beta}(x)
    \right\rangle
    d\mu_t(x)
    \notag\\
    &=
    \int_{\mathbb S^{d-1}}
    \lVert\mathcal X_{\nu,\beta}(x)\rVert_2^2
    d\mu_t(x)
    \notag\\
    &=
    \mathcal I_\beta[\mu_t|\nu]
    \geq
    0.
    \label{eq:app_fixed_energy_identity}
\end{align}
Thus, the conditional energy is nondecreasing along the flow. Its derivative vanishes at time $t$ if and only if
$\mathcal X_{\nu,\beta}=0$ holds $\mu_t$-almost everywhere.

\subsubsection{Proof of
Corollary~\ref{cor:explicit_global_convergence}}
\label{app:explicit_global_convergence}

\begin{proof}
Fix $x\in\mathbb S^{d-1}$ and write $x_t=\phi_t(x)$. Since $\dot{x}_t=\nabla^\circ\Phi_\nu(x_t)$,
\begin{equation}
\frac{d}{dt}\Phi_\nu(x_t)
=
\left\|
\nabla^\circ\Phi_\nu(x_t)
\right\|_2^2
\geq0.
\label{eq:explicit_global_energy_growth}
\end{equation}
Because $\mathbb S^{d-1}$ is compact, $\Phi_\nu(x_t)$ is bounded above and therefore converges. Integrating Eq.~\ref{eq:explicit_global_energy_growth} gives
\begin{equation}
\int_0^\infty
\left\|
\nabla^\circ\Phi_\nu(x_t)
\right\|_2^2
\,dt
<
\infty.
\label{eq:explicit_global_finite_dissipation}
\end{equation}

Since $\Phi_\nu$ is real analytic on $\mathbb S^{d-1}$, the Łojasiewicz gradient inequality ensures that every trajectory converges to a single critical point~\citep{absil2005convergence}. Consequently, $\phi_\infty(x)=\lim_{t\to\infty}\phi_t(x)\in\operatorname{Crit}(\Phi_\nu)$ is well defined for every $x\in\mathbb S^{d-1}$. Since each $\phi_t$ is continuous, the pointwise limit $\phi_\infty$ is Borel measurable.


For each group, the measure $(\phi_t,\phi_\infty)_\#\mu_0^a$ couples $\mu_t^a=(\phi_t)_\#\mu_0^a$ and $\mu_\infty^a=(\phi_\infty)_\#\mu_0^a$. Thus,
\begin{equation}
W_2^2(\mu_t^a,\mu_\infty^a)
\leq
\int_{\mathbb S^{d-1}}
d_{\mathbb S}
\left(
\phi_t(x),\phi_\infty(x)
\right)^2
\,d\mu_0^a(x).
\label{eq:explicit_global_groupwise_coupling}
\end{equation}
The integrand converges pointwise to zero and is bounded by $\pi^2$. Dominated convergence gives $W_2(\mu_t^a,\mu_\infty^a)\to0$ for every $a$.

Combining these groupwise couplings yields
\begin{equation}
W_2^2(\mu_t,\mu_\infty)
\leq
\sum_{a=1}^{m}
\omega_a
W_2^2(\mu_t^a,\mu_\infty^a)
\longrightarrow
0.
\end{equation}
Since $\phi_\infty$ takes values in $\operatorname{Crit}(\Phi_\nu)$, all limiting measures are supported on the critical set.
\end{proof}

\subsubsection{Proof of Theorem~\ref{thm:explicit_separated_attractors}}
\label{app:explicit_separated_attractors_proof}
\begin{proof}

For $a\neq b$, the triangle inequality gives
\begin{equation}
    \theta_{ab}
    =
    d_{\mathbb S}(\ell_a,\ell_b)
    \geq
    d_{\mathbb S}(c_a,c_b)
    -
    d_{\mathbb S}(\ell_a,c_a)
    -
    d_{\mathbb S}(\ell_b,c_b)
    \geq
    \Delta-2\eta
    >
    2\alpha.
    \label{eq:app_explicit_leader_separation}
\end{equation}
Thus, the caps $\mathcal S_\alpha^+(\ell_a)$ are pairwise disjoint. Moreover, for $x\in\operatorname{supp}(\mu_0^a)$,
\begin{equation}
    d_{\mathbb S}(x,\ell_a)
    \leq
    d_{\mathbb S}(x,c_a)
    +
    d_{\mathbb S}(c_a,\ell_a)
    \leq
    \delta+\eta
    <
    \alpha.
    \label{eq:app_explicit_initial_localization}
\end{equation}
Hence,
\begin{equation}
    \operatorname{supp}(\mu_0^a)
    \subseteq
    \mathcal S_{\delta+\eta}^+(\ell_a)
    \Subset
    \operatorname{int}\mathcal S_\alpha^+(\ell_a).
    \label{eq:app_explicit_compact_inclusion}
\end{equation}

For later use, if $x\in\mathcal S_\alpha^+(\ell_a)$ and $b\neq a$, then
\begin{equation}
    d_{\mathbb S}(x,\ell_b)
    \geq
    \theta_{ab}-d_{\mathbb S}(x,\ell_a)
    \geq
    \theta_{ab}-\alpha.
\end{equation}
Since $0<\theta_{ab}-\alpha<\pi$ and the cosine is decreasing on $[0,\pi]$,
\begin{equation}
    \langle x,\ell_b\rangle
    \leq
    \cos(\theta_{ab}-\alpha)
    =
    \gamma_{ab}(\alpha).
    \label{eq:app_explicit_cross_score}
\end{equation}

Fix $x\in\partial\mathcal S_\alpha^+(\ell_a)$, so that $\langle x,\ell_a\rangle=\cos\alpha$. The derivative of $h_a(x)=\langle x,\ell_a\rangle$ along the vector field is
\begin{align}
    D h_a(x)[\mathcal X_{\nu,\beta}(x)]
    &=
    \left\langle
        \mathcal X_{\nu,\beta}(x),
        \ell_a
    \right\rangle
    \notag\\
    &=
    \sum_{b=1}^{m}
    \rho_b r_b
    e^{\beta r_b\langle x,\ell_b\rangle}
    \left\langle
        \mathbf P_x^\perp\ell_b,
        \mathbf P_x^\perp\ell_a
    \right\rangle.
    \label{eq:app_explicit_boundary_derivative}
\end{align}
The term with $b=a$ is
\begin{equation}
    \rho_a r_a
    e^{\beta r_a\cos\alpha}
    \sin^2\alpha.
    \label{eq:app_explicit_self_inward}
\end{equation}
For $b\neq a$, Eq.~\ref{eq:app_explicit_cross_score} gives
\begin{align}
    \left\langle
        \mathbf P_x^\perp\ell_b,
        \mathbf P_x^\perp\ell_a
    \right\rangle
    &=
    \cos\theta_{ab}
    -
    \langle x,\ell_b\rangle\cos\alpha
    \notag\\
    &\geq
    \cos\theta_{ab}
    -
    \cos\alpha\cos(\theta_{ab}-\alpha)
    \notag\\
    &=
    -\sin\alpha\sin(\theta_{ab}-\alpha).
    \label{eq:app_explicit_cross_tangent}
\end{align}
If the projected inner product is nonnegative, its contribution to Eq.~\ref{eq:app_explicit_boundary_derivative} is nonnegative. If it is negative, Eqs.~\ref{eq:app_explicit_cross_score} and~\ref{eq:app_explicit_cross_tangent} imply
\begin{align}
    &\rho_b r_b
    e^{\beta r_b\langle x,\ell_b\rangle}
    \left\langle
        \mathbf P_x^\perp\ell_b,
        \mathbf P_x^\perp\ell_a
    \right\rangle
    \notag\\
    &\qquad\geq
    -
    \rho_b r_b
    e^{\beta r_b\gamma_{ab}(\alpha)}
    \sin\alpha
    \sin(\theta_{ab}-\alpha).
    \label{eq:app_explicit_cross_inward}
\end{align}
Summing over $b$ yields
\begin{equation}
    D h_a(x)[\mathcal X_{\nu,\beta}(x)]
    \geq
    \mathbf I_a(\alpha)
    >
    0.
    \label{eq:app_explicit_inward_bound}
\end{equation}
Thus, the vector field points strictly into the cap along its boundary, so $\mathcal S_\alpha^+(\ell_a)$ is forward invariant.

Write
\begin{equation}
    \Phi_\nu
    =
    \sum_{b=1}^{m}\Phi_b,
    \qquad
    \Phi_b(x)
    =
    \frac{\rho_b}{\beta}
    e^{\beta r_b\langle x,\ell_b\rangle}.
\end{equation}
For a unit vector $v\in T_x\mathbb S^{d-1}$,
\begin{equation}
    \nabla_{\mathbb S}^2\Phi_b(x)[v,v]
    =
    \rho_b r_b
    e^{\beta r_b\langle x,\ell_b\rangle}
    \left(
        \beta r_b\langle v,\ell_b\rangle^2
        -
        \langle x,\ell_b\rangle
    \right).
    \label{eq:app_explicit_component_hessian}
\end{equation}
For $x\in\mathcal S_\alpha^+(\ell_a)$,
\begin{equation}
    \langle x,\ell_a\rangle
    \geq
    \cos\alpha,
    \qquad
    |\langle v,\ell_a\rangle|
    \leq
    \lVert\mathbf P_x^\perp\ell_a\rVert_2
    \leq
    \sin\alpha.
\end{equation}
Since $\mathbf H_a(\alpha)>0$, we have $\cos\alpha-\beta r_a\sin^2\alpha>0$. Therefore,
\begin{equation}
    \nabla_{\mathbb S}^2\Phi_a(x)[v,v]
    \leq
    -
    \rho_a r_a
    e^{\beta r_a\cos\alpha}
    \left(
        \cos\alpha-\beta r_a\sin^2\alpha
    \right).
    \label{eq:app_explicit_self_hessian}
\end{equation}
For $b\neq a$, Eq.~\ref{eq:app_explicit_cross_score} and $|\langle x,\ell_b\rangle|\leq1$ give
\begin{equation}
    \nabla_{\mathbb S}^2\Phi_b(x)[v,v]
    \leq
    \rho_b r_b
    e^{\beta r_b\gamma_{ab}(\alpha)}
    (1+\beta r_b).
    \label{eq:app_explicit_cross_hessian}
\end{equation}
Indeed, this bound is immediate when the parenthesis in Eq.~\ref{eq:app_explicit_component_hessian} is nonpositive; otherwise, the exponential is bounded using Eq.~\ref{eq:app_explicit_cross_score}. Summing Eqs.~\ref{eq:app_explicit_self_hessian} and~\ref{eq:app_explicit_cross_hessian} yields
\begin{equation}
    \nabla_{\mathbb S}^2\Phi_\nu(x)[v,v]
    \leq
    -\mathbf H_a(\alpha)
    <
    0,
    \qquad
    x\in\mathcal S_\alpha^+(\ell_a).
    \label{eq:app_explicit_strong_concavity}
\end{equation}

Because $\alpha<\pi/2$, the cap $\mathcal S_\alpha^+(\ell_a)$ is geodesically convex. Compactness ensures that $\Phi_\nu$ attains a maximum on the cap. At a boundary point, Eq.~\ref{eq:app_explicit_inward_bound} makes $\mathcal X_{\nu,\beta}$ a feasible ascent direction, and
\begin{equation}
    D\Phi_\nu[\mathcal X_{\nu,\beta}]
    =
    \lVert\mathcal X_{\nu,\beta}\rVert_2^2
    >
    0.
\end{equation}
Thus, the maximum lies in the interior and is a critical point. The uniform concavity in Eq.~\ref{eq:app_explicit_strong_concavity} makes this point unique; denote it by $u_a$. The same bound shows that $u_a$ is a nondegenerate local maximizer.

We next derive the contraction rate. Fix $x\in\mathcal S_\alpha^+(\ell_a)$ and let $\gamma:[0,1]\to\mathcal S_\alpha^+(\ell_a)$ be the constant-speed minimizing geodesic from $x$ to $u_a$. Set $g(s)=\Phi_\nu(\gamma(s))$. Since $\lVert\dot\gamma(s)\rVert_2=d_{\mathbb S}(x,u_a)$, Eq.~\ref{eq:app_explicit_strong_concavity} gives
\begin{equation}
    g''(s)
    \leq
    -\mathbf H_a(\alpha)
    d_{\mathbb S}(x,u_a)^2.
\end{equation}
Moreover, $g'(1)=0$ because $\nabla^\circ\Phi_\nu(u_a)=0$. Since $\operatorname{Log}_x(u_a)=\dot\gamma(0)$, integration over $s\in[0,1]$ gives
\begin{equation}
    \left\langle
        \mathcal X_{\nu,\beta}(x),
        \operatorname{Log}_x(u_a)
    \right\rangle
    =
    g'(0)
    \geq
    \mathbf H_a(\alpha)
    d_{\mathbb S}(x,u_a)^2.
    \label{eq:app_explicit_gradient_coercivity}
\end{equation}
Let $x_t=\varphi_t(x)$. Forward invariance and the first variation of the squared distance yield
\begin{align}
    \frac{1}{2}\frac{d}{dt}
    d_{\mathbb S}(x_t,u_a)^2
    &=
    -
    \left\langle
        \mathcal X_{\nu,\beta}(x_t),
        \operatorname{Log}_{x_t}(u_a)
    \right\rangle
    \notag\\
    &\leq
    -\mathbf H_a(\alpha)
    d_{\mathbb S}(x_t,u_a)^2.
\end{align}
Gr\"onwall's inequality therefore gives
\begin{equation}
    d_{\mathbb S}(\varphi_t(x),u_a)
    \leq
    e^{-\mathbf H_a(\alpha)t}
    d_{\mathbb S}(x,u_a),
    \qquad
    t\geq0.
    \label{eq:app_explicit_pointwise_contraction}
\end{equation}
Thus, every point in $\mathcal S_\alpha^+(\ell_a)$ converges to $u_a$, and this cap is contained in the attraction basin $B_a$ of $u_a$.

Since $u_a\in\mathcal S_\alpha^+(\ell_a)$ and $d_{\mathbb S}(\ell_a,c_a)\leq\eta$,
\begin{equation}
    d_{\mathbb S}(u_a,\ell_a)
    \leq
    \alpha,
    \qquad
    d_{\mathbb S}(u_a,c_a)
    \leq
    \alpha+\eta.
\end{equation}
Hence, for $a\neq b$,
\begin{align}
    d_{\mathbb S}(u_a,u_b)
    &\geq
    d_{\mathbb S}(c_a,c_b)
    -
    d_{\mathbb S}(u_a,c_a)
    -
    d_{\mathbb S}(u_b,c_b)\notag
    \\
    &\geq
    \Delta-2\eta-2\alpha
    >
    0.
\end{align}
Thus, the attractors are pairwise distinct. The boundary and strict concavity arguments above also show that $u_a$ is the only critical point in $\mathcal S_\alpha^+(\ell_a)$. Since the caps are disjoint, their union contains exactly the $m$ critical points $u_1,\ldots,u_m$.

For each $a$, forward invariance gives
\begin{equation}
    \operatorname{supp}(\mu_t^a)
    \subseteq
    \mathcal S_\alpha^+(\ell_a),
    \qquad
    t\geq0.
\end{equation}
Because the second marginal is the Dirac measure $\delta_{u_a}$, the pointwise contraction estimate above yields
\begin{align}
    W_2^2(\mu_t^a,\delta_{u_a})
    &=
    \int_{\mathbb S^{d-1}}
    d_{\mathbb S}(\varphi_t(x),u_a)^2
    d\mu_0^a(x) \notag
    \\
    &\leq
    e^{-2\mathbf H_a(\alpha)t}
    \int_{\mathbb S^{d-1}}
    d_{\mathbb S}(x,u_a)^2
    d\mu_0^a(x) \notag
    \\
    &=
    D_{0,a}^2e^{-2\mathbf H_a(\alpha)t}.
\end{align}
This proves the groupwise estimate in Eq.~\ref{eq:explicit_groupwise_rate}.

By linearity of the pushforward,
\begin{equation}
    \mu_t
    =
    \sum_{a=1}^{m}
    \omega_a\mu_t^a.
\end{equation}
Combining the groupwise couplings gives
\begin{align}
    W_2^2(\mu_t,\mu_\infty)
    &\leq
    \sum_{a=1}^{m}
    \omega_a
    W_2^2(\mu_t^a,\delta_{u_a}) \notag
    \\
    &\leq
    \sum_{a=1}^{m}
    \omega_a
    D_{0,a}^2
    e^{-2\mathbf H_a(\alpha)t} \notag
    \\
    &\leq
    D_0^2e^{-2\mathbf H_{\min}(\alpha)t}.
\end{align}
Taking square roots proves the first inequality in Eq.~\ref{eq:explicit_quantitative_rate}. Moreover, for $x\in\operatorname{supp}(\mu_0^a)$,
\begin{align}
    d_{\mathbb S}(x,u_a)
    &\leq
    d_{\mathbb S}(x,c_a)
    +
    d_{\mathbb S}(c_a,\ell_a)
    +
    d_{\mathbb S}(\ell_a,u_a) \notag
    \\
    &\leq
    \delta+\eta+\alpha.
\end{align}
Therefore, $D_0\leq\alpha+\delta+\eta$. Finally, since every $\omega_a>0$ and the points $u_1,\ldots,u_m$ are pairwise distinct,
\begin{equation}
    \operatorname{supp}(\mu_\infty)
    =
    \{u_1,\ldots,u_m\},
    \qquad
    \left|
        \operatorname{supp}(\mu_\infty)
    \right|
    =
    m,
\end{equation}
which completes the proof.
\end{proof}

\subsection{Implicit Opinion Leader Dynamics}
\label{app:implicit_dynamics}





\subsubsection{Radius-graph decomposition}
\label{app:neighborhood_component_decomposition}

We first verify the component structure used in Sec.~\ref{subsec:neighborhood_dynamics}. Suppose, for contradiction, that $G_R$ is connected. Then it contains a spanning tree $T_R\in\mathbb T_n$. Since every edge of $T_R$ also belongs to $G_R$, its length is at most $R$. Hence,
\begin{equation}
    R_{\mathrm{conn}}
    \leq
    \max_{\{i,j\}\in \mathcal E(T_R)}
    d_{\mathcal Z}(p_i,p_j)
    \leq R,
\end{equation}
which contradicts $R<R_{\mathrm{conn}}$. Therefore, $G_R$ is disconnected, and its reachability classes satisfy $K\geq2$.

For any $i\in\{1,\ldots,n\}$, the definition of $R_{\mathrm{loc}}$ ensures that there exists $j\neq i$ such that
\begin{equation}
    d_{\mathcal Z}(p_i,p_j)
    \leq
    R_{\mathrm{loc}}
    \leq R.
\end{equation}
Thus, every node has at least one distinct neighbor, so each connected component contains at least two nodes.

Finally, let $i\in\mathcal I_k$ and $j\in\mathcal I_\ell$ with $k\neq\ell$. Then $d_{\mathcal Z}(p_i,p_j)>R$; otherwise, $\{i,j\}$ would be an edge of $G_R$, implying $i\sim_R j$ and hence $k=\ell$. Therefore, $\mathcal N_i^{(R)}\subseteq\mathcal I_k$, $i\in\mathcal I_k$. After a permutation of the nodes, the adjacency matrix is block diagonal with blocks corresponding to $\{G_R[\mathcal I_k]\}_{k=1}^{K}$. Consequently, the dynamics decompose into $K$ independent graph-coupled systems, one on each connected component.

\subsubsection{Graph Structure and Product Wasserstein Flow}
\label{app:neighborhood_wasserstein_structure}
\begin{lemma}[\textbf{Well-posedness of the graph-coupled flow}]
\label{lem:neighborhood_wellposedness}
For every initial tuple
\begin{equation}
    \boldsymbol{\nu}_0^k
    =
    (\nu_{i,0}^k)_{i\in\mathcal I_k}
    \in
    \mathcal P_2(\mathbb S^{d-1})^{\mathcal I_k},
\end{equation}
Eq.~\ref{eq:neighborhood_product_continuity_equation} admits a unique global characteristic solution
\begin{equation}
    \boldsymbol{\nu}^k
    \in
    C\left(
        [0,\infty);
        \mathcal P_2(\mathbb S^{d-1})^{\mathcal I_k}
    \right).
\end{equation}
This solution is locally absolutely continuous with respect to $\mathcal W_{2,\pi^k}$, and the energy and moment identities used below hold along the flow.
\end{lemma}

\begin{proof}
Define
\begin{equation}
    \mathcal K_\beta(x,y)
    =
    \mathbf P_x^\perp(y)
    e^{\beta\langle x,y\rangle}.
\end{equation}
The kernel $\mathcal K_\beta$ is smooth on the compact set $\mathbb S^{d-1}\times\mathbb S^{d-1}$, and is therefore uniformly bounded and Lipschitz in both variables. It follows that there exist constants $M_\beta,L_\beta>0$ such that
\begin{equation}
    \sup_{x\in\mathbb S^{d-1}}
    \left\|
    \mathcal X_{i,\boldsymbol{\nu}^k,\beta}(x)
    \right\|_2
    \leq M_\beta,
    \qquad
    \left\|
    \mathcal X_{i,\boldsymbol{\nu}^k,\beta}(x)
    -
    \mathcal X_{i,\boldsymbol{\nu}^k,\beta}(x')
    \right\|_2
    \leq
    L_\beta d_{\mathbb S}(x,x').
\end{equation}
Moreover, for two measure tuples $\boldsymbol{\nu}^k$ and $\boldsymbol{\eta}^k$,
\begin{equation}
    \sup_{x\in\mathbb S^{d-1}}
    \left\|
    \mathcal X_{i,\boldsymbol{\nu}^k,\beta}(x)
    -
    \mathcal X_{i,\boldsymbol{\eta}^k,\beta}(x)
    \right\|_2
    \leq
    L_\beta
    \sum_{j\in\mathcal I_k}
    \mathsf P_{ij}^k
    W_2(\nu_j^k,\eta_j^k).
\end{equation} 
For $T>0$, consider the complete metric space
\begin{equation}
    \mathcal C_T
    =
    \{
    \boldsymbol{\zeta}^k
    \in
    C\left(
    [0,T];
    \mathcal P_2(\mathbb S^{d-1})^{\mathcal I_k}
    \right)
    :
    \boldsymbol{\zeta}_0^k
    =
    \boldsymbol{\nu}_0^k
    \},
\end{equation}
equipped with
\begin{equation}
    d_T
    \left(
    \boldsymbol{\zeta}^k,
    \boldsymbol{\eta}^k
    \right)
    =
    \sup_{0\leq t\leq T}
    \mathcal W_{2,\pi^k}
    \left(
    \boldsymbol{\zeta}_t^k,
    \boldsymbol{\eta}_t^k
    \right).
\end{equation}
For $\boldsymbol{\zeta}^k\in\mathcal C_T$, let
$\Phi_{i,t}^{\boldsymbol{\zeta}^k}$ denote the characteristic flow generated by $\mathcal X_{i,\boldsymbol{\zeta}_t^k,\beta}$, and define
\begin{equation}
    \left(
    \mathcal T\boldsymbol{\zeta}^k
    \right)_{i,t}
    =
    \left(
    \Phi_{i,t}^{\boldsymbol{\zeta}^k}
    \right)_{\#}
    \nu_{i,0}^k.
\end{equation}
Thus, $\mathcal T$ maps $\mathcal C_T$ into itself. The estimates above, Jensen's inequality, the stationarity identity
\begin{equation}
    \sum_{i\in\mathcal I_k}
    \pi_i^k\mathsf P_{ij}^k
    =
    \pi_j^k,
\end{equation}
and Grönwall's inequality give
\begin{equation}
    d_T
    \left(
    \mathcal T\boldsymbol{\zeta}^k,
    \mathcal T\boldsymbol{\eta}^k
    \right)
    \leq
    L_\beta T e^{L_\beta T}
    d_T
    \left(
    \boldsymbol{\zeta}^k,
    \boldsymbol{\eta}^k
    \right).
\end{equation}
For sufficiently small $T>0$, the map $\mathcal T$ is therefore a contraction. The Banach fixed-point theorem gives local existence and uniqueness. Since $M_\beta$ and $L_\beta$ are uniform in time, the same argument can be repeated on consecutive time intervals, yielding a unique global solution. 

The characteristic representation also gives
\begin{equation}
    \left\|
    \dot{\boldsymbol{\nu}}_t^k
    \right\|_{\mathcal W_{2,\pi^k}}^2
    \leq
    \sum_{i\in\mathcal I_k}
    \pi_i^k
    \int_{\mathbb S^{d-1}}
    \left|
    \mathcal X_{i,\boldsymbol{\nu}_t^k,\beta}(x)
    \right|_2^2
    d\nu_{i,t}^k(x)
    \leq
    M_\beta^2.
\end{equation}
Hence, the solution is locally absolutely continuous with respect to $\mathcal W_{2,\pi^k}$. Testing the continuity equations against smooth functions and using the smoothness of $\mathcal K_\beta$ then yields the energy and moment identities used below.
\end{proof}

Fix a component $\mathcal I_k$. Define the reversible edge weights by
\begin{equation}
    q_{ij}^k
    =
    \pi_i^k\mathsf P_{ij}^k
    =
    \frac{A_{ij}^{(R)}}{\operatorname{vol}_k}
    =
    \pi_j^k\mathsf P_{ji}^k
    =
    q_{ji}^k.
    \label{eq:app_graph_detailed_balance}
\end{equation}
Thus, $\mathsf P^k$ is reversible with respect to $\pi^k$. For $\mathbf z=(z_i)_{i\in\mathcal I_k}$, let $\overline z_{\pi^k}=\sum_{i\in\mathcal I_k}\pi_i^k z_i$ and define the graph Dirichlet form
\begin{equation}
    \mathcal D_k(\mathbf z)
    =
    \frac{1}{2}
    \sum_{i,j\in\mathcal I_k}
    \pi_i^k\mathsf P_{ij}^k
    \|
        z_i-z_j
    \|_2^2.
    \label{eq:app_graph_dirichlet_form}
\end{equation}
Since the normalized Laplacian is similar to
$\mathbf I_{n_k}-\mathsf P_k$, the spectral gap $\lambda_k$ yields the Poincar\'e inequality
\begin{equation}
    \lambda_k
    \sum_{i\in\mathcal I_k}
    \pi_i^k
    \|
        z_i-\overline z_{\pi^k}
    \|_2^2
    \leq
    \mathcal D_k(\mathbf z)
    \leq
    2
    \sum_{i\in\mathcal I_k}
    \pi_i^k
    \|
        z_i-\overline z_{\pi^k}
    \|_2^2.
    \label{eq:app_graph_poincare}
\end{equation}
The upper bound follows because the spectrum of the normalized Laplacian is contained in $[0,2]$. The vector-valued inequality follows by applying the scalar inequality coordinatewise.

Because every node has a self-loop, $\mathsf P_{ii}^k=1/d_i^{(R)}$. Moreover, $d_i^{(R)}\leq n_k$, and hence, $\operatorname{tr}(\mathsf P_k)=\sum_{i\in\mathcal I_k}\frac{1}{d_i^{(R)}}\geq1$.

Let $1=\theta_1^k\geq\theta_2^k\geq\cdots\geq\theta_{n_k}^k$ be the eigenvalues of $\mathsf P_k$. Since $\sum_{r=2}^{n_k}\theta_r^k\geq0$, we have $\theta_2^k\geq0$. Therefore,
\begin{equation}
    0<\lambda_k=1-\theta_2^k\leq1.
    \label{eq:app_spectral_gap_upper_bound}
\end{equation}

The graph interaction energy is
\begin{equation}
    \mathcal E_{\beta,k}^{G}
    \left[
        \boldsymbol\nu^k
    \right]
    =
    \frac{1}{2\beta}
    \sum_{i,j\in\mathcal I_k}
    q_{ij}^k
    \iint_{\mathbb S^{d-1}\times\mathbb S^{d-1}}
    e^{\beta\langle x,y\rangle}
    \,d\nu_i^k(x)d\nu_j^k(y).
    \label{eq:neighborhood_product_energy}
\end{equation}

We next verify the product Wasserstein gradient structure. By symmetry of $q_{ij}^k$, the first variation of $\mathcal E_{\beta,k}^{G}$ with respect to the $i$-th species is
\begin{equation}
    \frac{\delta\mathcal E_{\beta,k}^{G}}
    {\delta\nu_i^k}(x)
    =
    \frac{1}{\beta}
    \sum_{j\in\mathcal I_k}
    q_{ij}^k
    \int_{\mathbb S^{d-1}}
    e^{\beta\langle x,y\rangle}
    \,d\nu_j^k(y).
    \label{eq:app_product_first_variation}
\end{equation}
Since $q_{ij}^k=\pi_i^k\mathsf P_{ij}^k$, taking the spherical gradient gives
\begin{align}
    \nabla_x^\circ
    \left(
        \frac{1}{\pi_i^k}
        \frac{\delta\mathcal E_{\beta,k}^{G}}
        {\delta\nu_i^k}(x)
    \right)
    &=
    \sum_{j\in\mathcal I_k}
    \mathsf P_{ij}^k
    \int_{\mathbb S^{d-1}}
    \mathbf P_x^\perp(y)
    e^{\beta\langle x,y\rangle}
    \,d\nu_j^k(y)
    \notag\\
    &=
    \mathcal X_{i,\boldsymbol\nu^k,\beta}(x).
    \label{eq:app_product_wasserstein_gradient}
\end{align}

To verify this identification, let $\xi_i\in L^2(\nu_i^k;T\mathbb S^{d-1})$ and consider the perturbation $\nu_{i,\varepsilon}^k=(\exp_x(\varepsilon\xi_i(x)))_{\#}\nu_i^k$.
Then
\begin{equation}
    \left.
    \frac{d}{d\varepsilon}
    \mathcal E_{\beta,k}^{G}
    \left[
        \boldsymbol\nu_\varepsilon^k
    \right]
    \right|_{\varepsilon=0}
    =
    \sum_{i\in\mathcal I_k}
    \pi_i^k
    \int_{\mathbb S^{d-1}}
    \left\langle
        \mathcal X_{i,\boldsymbol\nu^k,\beta}(x),
        \xi_i(x)
    \right\rangle
    \,d\nu_i^k(x).
    \label{eq:app_product_directional_derivative}
\end{equation}
Hence, $\left(\mathcal X_{i,\boldsymbol\nu^k,\beta}\right)_{i\in\mathcal I_k}$ is the \emph{Wasserstein gradient} of $\mathcal E_{\beta,k}^{G}$ under $\mathcal W_{2,\pi^k}$.

Define the squared Wasserstein slope by
\begin{equation}
    \mathcal I_{\beta,k}^{G}
    \left[
        \boldsymbol\nu^k
    \right]
    =
    \sum_{i\in\mathcal I_k}
    \pi_i^k
    \int_{\mathbb S^{d-1}}
    \|
        \mathcal X_{i,\boldsymbol\nu^k,\beta}(x)
    \|_2^2
    \,d\nu_i^k(x).
    \label{eq:neighborhood_product_slope}
\end{equation}
Using the continuity equation in Eq.~\ref{eq:neighborhood_product_continuity_equation} and spherical integration by parts, we obtain
\begin{align}
    \frac{d}{dt}
    \mathcal E_{\beta,k}^{G}
    \left[
        \boldsymbol\nu_t^k
    \right]
    &=
    \sum_{i\in\mathcal I_k}
    \pi_i^k
    \int_{\mathbb S^{d-1}}
    \|
        \mathcal X_{i,\boldsymbol\nu_t^k,\beta}(x)
    \|_2^2
    \,d\nu_{i,t}^k(x)
    \notag\\
    &=
    \mathcal I_{\beta,k}^{G}
    \left[
        \boldsymbol\nu_t^k
    \right]
    \geq0.
    \label{eq:app_product_energy_identity}
\end{align}
Thus, $\boldsymbol\nu_t^k$ is a \emph{reverse Wasserstein gradient flow} of $\mathcal E_{\beta,k}^{G}$ in the weighted product space.

\subsubsection{Proof of Theorem~\ref{thm:neighborhood_multicluster}}
\label{app:neighborhood_proof}
\begin{proof}
By Assumption~\ref{ass:neighborhood_mask} and the radius-graph decomposition in Appendix~\ref{app:neighborhood_component_decomposition}, the sets $\{\mathcal I_k\}_{k=1}^{K}$ are the nontrivial connected components of $G_R$. Assumption~\ref{ass:neighborhood_initial_states} further provides pairwise disjoint initial caps. Fix a component $\mathcal I_k$ and set $c_\alpha=\cos(2\alpha)>0$. We first prove forward invariance. The case $\alpha=0$ is immediate, so suppose that $\alpha>0$. Let $x\in\partial S_\alpha^+(u_k)$ and $y\in S_\alpha^+(u_k)$. Write
\begin{equation}
    x
    =
    \cos\alpha\,u_k
    +
    \sin\alpha\,v,
    \qquad
    y
    =
    \cos\theta\,u_k
    +
    \sin\theta\,w,
\end{equation}
where $\theta\leq\alpha$ and
$v,w\in u_k^\perp$ are unit vectors. Then
\begin{align}
    \left\langle
        \mathbf P_x^\perp(y),
        u_k
    \right\rangle
    &=
    \langle y,u_k\rangle
    -
    \langle x,y\rangle
    \langle x,u_k\rangle
    \notag\\
    &=
    \sin\alpha
    \left(
        \sin\alpha\cos\theta
        -
        \cos\alpha\sin\theta
        \langle v,w\rangle
    \right)
    \notag\\
    &\geq
    \sin\alpha\sin(\alpha-\theta)
    \geq0.
    \label{eq:app_measure_cap_inward}
\end{align}
Therefore,
\begin{equation}
    \left\langle
        \mathcal X_{i,\boldsymbol\nu_t^k,\beta}(x),
        u_k
    \right\rangle
    \geq0
\end{equation}
for every boundary point $x$ whenever all nodewise measures are supported in the cap. The characteristic flow associated with Eq.~\ref{eq:neighborhood_product_continuity_equation} consequently preserves the cap:
\begin{equation}
    \operatorname{supp}
    \left(
        \nu_{i,t}^k
    \right)
    \subseteq
    S_\alpha^+(u_k),
    \qquad
    t\geq0,
    \quad
    i\in\mathcal I_k.
    \label{eq:app_measure_cap_invariance}
\end{equation}

Define the degree-weighted mean by
\begin{equation}
    \overline x_t^k
    =
    \sum_{i\in\mathcal I_k}
    \pi_i^k
    \int_{\mathbb S^{d-1}}
    x
    \,d\nu_{i,t}^k(x)
    =
    \int_{\mathbb S^{d-1}}
    x
    \,d\widehat\mu_t^k(x),
    \label{eq:neighborhood_weighted_mean}
\end{equation}
and define the degree-weighted disagreement by
\begin{align}
    V_t^k
    &=
    1-
    \|
        \overline x_t^k
    \|_2^2
    \notag\\
    &=
    \int_{\mathbb S^{d-1}}
    \|
        x-\overline x_t^k
    \|_2^2
    \,d\widehat\mu_t^k(x)
    \notag\\
    &=
    \frac{1}{2}
    \iint_{\mathbb S^{d-1}\times\mathbb S^{d-1}}
    \|x-y\|_2^2
    \,d\widehat\mu_t^k(x)
    d\widehat\mu_t^k(y).
    \label{eq:neighborhood_weighted_disagreement}
\end{align}

Define the edge-weighted disagreement by
\begin{equation}
    \mathcal D_k
    \left[
        \boldsymbol\nu_t^k
    \right]
    =
    \frac{1}{2}
    \sum_{i,j\in\mathcal I_k}
    q_{ij}^k
    \iint_{\mathbb S^{d-1}\times\mathbb S^{d-1}}
    \|x-y\|_2^2
    \,d\nu_{i,t}^k(x)
    d\nu_{j,t}^k(y).
    \label{eq:app_measure_graph_disagreement}
\end{equation}
Differentiating the weighted mean through the continuity equations gives
\begin{equation}
    \frac{d}{dt}
    \overline x_t^k
    =
    \sum_{i\in\mathcal I_k}
    \pi_i^k
    \int_{\mathbb S^{d-1}}
    \mathcal X_{i,\boldsymbol\nu_t^k,\beta}(x)
    \,d\nu_{i,t}^k(x).
\end{equation}
Using $q_{ij}^k=q_{ji}^k$ and symmetrizing in
$(i,x)$ and $(j,y)$, we obtain
\begin{align}
    \frac{d}{dt}V_t^k
    &=
    -2
    \left\langle
        \overline x_t^k,
        \frac{d}{dt}
        \overline x_t^k
    \right\rangle
    \notag\\
    &=
    -
    \sum_{i,j\in\mathcal I_k}
    q_{ij}^k
    \iint_{\mathbb S^{d-1}\times\mathbb S^{d-1}}
    e^{\beta\langle x,y\rangle}
    \bigl(
        1-\langle x,y\rangle
    \bigr)
    \left\langle
        \overline x_t^k,
        x+y
    \right\rangle
    \,d\nu_{i,t}^k(x)
    d\nu_{j,t}^k(y).
    \label{eq:app_measure_disagreement_identity}
\end{align}
Forward invariance implies
\begin{equation}
    \langle x,y\rangle
    \geq
    c_\alpha,
    \qquad
    \left\langle
        \overline x_t^k,
        x+y
    \right\rangle
    \geq
    2c_\alpha
\end{equation}
on the supports of the nodewise measures. Hence,
\begin{equation}
    \frac{d}{dt}V_t^k
    \leq
    -2c_\alpha e^{\beta c_\alpha}
    \mathcal D_k
    \left[
        \boldsymbol\nu_t^k
    \right].
    \label{eq:app_measure_disagreement_decay_pre}
\end{equation}

To apply the graph Poincar\'e inequality, define
\begin{equation}
    \mathcal{M}_{i,t}^k
    =
    \int_{\mathbb S^{d-1}}
    x
    \,d\nu_{i,t}^k(x),
    \qquad
    \sigma_{i,t}^{k\,2}
    =
    \int_{\mathbb S^{d-1}}
    \|
        x-\mathcal{M}_{i,t}^k
    \|_2^2
    \,d\nu_{i,t}^k(x).
\end{equation}
Then
\begin{equation}
    V_t^k
    =
    \sum_{i\in\mathcal I_k}
    \pi_i^k
    \sigma_{i,t}^{k\,2}
    +
    \sum_{i\in\mathcal I_k}
    \pi_i^k
    \|
        \mathcal{M}_{i,t}^k-\overline x_t^k
    \|_2^2,
    \label{eq:app_measure_variance_decomposition}
\end{equation}
while
\begin{align}
    \mathcal D_k
    \left[
        \boldsymbol\nu_t^k
    \right]
    &=
    \sum_{i\in\mathcal I_k}
    \pi_i^k
    \sigma_{i,t}^{k\,2}
    \notag\\
    &\quad+
    \frac{1}{2}
    \sum_{i,j\in\mathcal I_k}
    q_{ij}^k
    \|
        \mathcal{M}_{i,t}^k-\mathcal{M}_{j,t}^k
    \|_2^2.
    \label{eq:app_measure_dirichlet_decomposition}
\end{align}
Applying the graph Poincar\'e inequality to
$(\mathcal{M}_{i,t}^k)_{i\in\mathcal I_k}$ and using
$\lambda_k\leq1$ gives
\begin{equation}
    \mathcal D_k
    \left[
        \boldsymbol\nu_t^k
    \right]
    \geq
    \lambda_kV_t^k.
    \label{eq:app_measure_poincare}
\end{equation}
Therefore,
\begin{equation}
    \frac{d}{dt}V_t^k
    \leq
    -2\lambda_k
    \cos(2\alpha)
    e^{\beta\cos(2\alpha)}
    V_t^k
    =
    -2\kappa_kV_t^k.
\end{equation}
Gronwall's inequality yields
\begin{equation}
    V_t^k
    \leq
    V_0^k
    e^{-2\kappa_k t}.
    \label{eq:app_measure_disagreement_rate}
\end{equation}

We next estimate the metric speed of the product Wasserstein flow. Using Jensen's inequality, $\langle x,y\rangle\leq1$, and $\|\mathbf P_x^\perp(y)\|_2\leq\|x-y\|_2$, we obtain
\begin{align}
    \mathcal I_{\beta,k}^{G}
    \left[
        \boldsymbol\nu_t^k
    \right]
    &\leq
    e^{2\beta}
    \sum_{i,j\in\mathcal I_k}
    q_{ij}^k
    \iint_{\mathbb S^{d-1}\times\mathbb S^{d-1}}
    \|x-y\|_2^2
    \,d\nu_{i,t}^k(x)
    d\nu_{j,t}^k(y)
    \notag\\
    &=
    2e^{2\beta}
    \mathcal D_k
    \left[
        \boldsymbol\nu_t^k
    \right].
    \label{eq:app_measure_slope_dirichlet}
\end{align}
The upper graph Poincar\'e bound and
Eq.~\ref{eq:app_measure_dirichlet_decomposition} give
\begin{equation}
    \mathcal D_k
    \left[
        \boldsymbol\nu_t^k
    \right]
    \leq
    2V_t^k.
\end{equation}
Consequently,
\begin{equation}
    \left|
        \dot{\boldsymbol\nu}_t^k
    \right|_{\mathcal W_{2,\pi^k}}
    \leq
    \left(
        \mathcal I_{\beta,k}^{G}
        \left[
            \boldsymbol\nu_t^k
        \right]
    \right)^{1/2}
    \leq
    2e^\beta
    \left(
        V_0^k
    \right)^{1/2}
    e^{-\kappa_k t}.
    \label{eq:app_measure_metric_speed}
\end{equation}
Therefore, for $s>t$,
\begin{align}
    \mathcal W_{2,\pi^k}
    \left(
        \boldsymbol\nu_t^k,
        \boldsymbol\nu_s^k
    \right)
    &\leq
    \int_t^s
    \left|
        \dot{\boldsymbol\nu}_r^k
    \right|_{\mathcal W_{2,\pi^k}}
    \,dr
    \notag\\
    &\leq
    \frac{2e^\beta}{\kappa_k}
    \left(
        V_0^k
    \right)^{1/2}
    e^{-\kappa_k t}.
    \label{eq:app_measure_tail_length}
\end{align}

The product Wasserstein space is complete, so
$\boldsymbol\nu_t^k$ converges to a limit
$\boldsymbol\nu_\infty^k$. Since $V_t^k\to0$, the aggregate limit has zero variance and must be a Dirac mass. Because every $\pi_i^k>0$, there exists $x_\infty^k\in S_\alpha^+(u_k)$ such that
\begin{equation}
    \boldsymbol\nu_\infty^k
    =
    \left(
        \delta_{x_\infty^k}
    \right)_{i\in\mathcal I_k}.
\end{equation}
Letting $s\to\infty$ in
Eq.~\ref{eq:app_measure_tail_length} and using $C_k=2e^\beta/\kappa_k$, we obtain
\begin{align}
    \mathcal W_{2,\pi^k}
    \left(
        \boldsymbol\nu_t^k,
        \boldsymbol\nu_\infty^k
    \right)
    &=
    W_2
    \left(
        \widehat\mu_t^k,
        \delta_{x_\infty^k}
    \right)
    \notag\\
    &\leq
    C_k
    \left(
        V_0^k
    \right)^{1/2}
    e^{-\kappa_k t}.
    \label{eq:app_measure_product_rate}
\end{align}

Since $1/n_k\leq\chi_k\pi_i^k$ for every $i\in\mathcal I_k$, we also have
\begin{equation}
    W_2
    \left(
        \mu_t^k,
        \delta_{x_\infty^k}
    \right)
    \leq
    \sqrt{\chi_k}\,
    C_k
    \left(
        V_0^k
    \right)^{1/2}
    e^{-\kappa_k t}.
    \label{eq:app_measure_uniform_rate}
\end{equation}

For $k\neq\ell$, cap separation gives $d_{\mathbb S}(x_\infty^k,x_\infty^\ell)\geq d_{\mathbb S}(u_k,u_\ell)-2\alpha>0$. Hence, the limiting directions are pairwise distinct.

Finally, combining the componentwise couplings gives
\begin{align}
    W_2^2
    \left(
        \mu_t,
        \mu_\infty
    \right)
    &\leq
    \sum_{k=1}^{K}
    m_k
    W_2^2
    \left(
        \mu_t^k,
        \delta_{x_\infty^k}
    \right)
    \notag\\
    &\leq
    \sum_{k=1}^{K}
    m_k\chi_kC_k^2V_0^k
    e^{-2\kappa_k t}
    \notag\\
    &\leq
    C_G^2
    e^{-2\kappa_{\min}t}.
    \label{eq:app_measure_global_rate}
\end{align}
Taking square roots completes the proof.
\end{proof}

\section{Extended Experiment Results}
\label{app:extended_results}

\subsection{Implementation Details}
\label{app:implementation_details}

\subsubsection{Numerical Simulations}
\label{app:simulation_implementation}

Table~\ref{tab:simulation_parameters} summarizes the implementation details for all four attention dynamics, including the shared numerical setup, initialization schemes, representative configurations for the \emph{explicit opinion leader dynamics}, and graph configurations for the \emph{implicit opinion leader dynamics}.

\begin{table*}[htbp]
\centering
\small
\caption{Implementation details of the numerical simulations.}
\label{tab:simulation_parameters}
\setlength{\tabcolsep}{5pt}

\resizebox{\textwidth}{!}{%
\begin{tabular}{l|cccc}
\toprule
\multirow{2}{*}{\textbf{Hyperparameters}}
& \multicolumn{4}{c}{\textbf{Attention dynamics}} \\
\cmidrule(l){2-5}
& \textbf{Standard}
& \textbf{Explicit (General)}
& \textbf{Explicit (Conditional)}
& \textbf{Implicit} \\
\midrule

Number of particles $n$
& 180 & 180 & 180 & 180 \\

Inverse temperature $\beta$
& 1.0 & 1.0 & 1.0 & 1.0 \\

Integrator
& RK4 & RK4 & RK4 & RK4 \\

Simulation Time $T$
& 10 & 10 & 0.05 & 10 \\

Time step $\Delta t$
& 0.05 & 0.05 & 0.01 & 0.05 \\
\midrule

Initialization
& $\mathrm{vMF}(e_1,\kappa=1)$
& Uniform on $\mathbb S^2$
& Grouped spherical caps
& Grouped spherical caps \\

vMF concentration $\kappa$
& 1.0 & \xmark & \xmark & \xmark \\

Number of representatives $m$
& \xmark & 6 & 6 & \xmark \\

Representative directions
& \xmark
& $\{\pm e_1,\pm e_2,\pm e_3\}$
& $\{\pm e_1,\pm e_2,\pm e_3\}$
& \xmark \\
Component-center directions
& \xmark
& \xmark
& \xmark
& $\{\pm e_1,\pm e_2,\pm e_3\}$ \\
Representative radii $r_a$
& \xmark
& $1.85\;(\times 6)$
& \texttt{auto} ($r_a \approx 6.0468$)
& \xmark \\

Number of groups or components
& \xmark & \xmark & 6 & 6 \\
Group or component sizes
& \xmark & \xmark & $30\;(\times 6)$ & $30\;(\times 6)$ \\
Initial cap radius
& \xmark & \xmark & $\delta=0.25$ & $\alpha=0.75$ \\
Explicit invariant-cap radius $\alpha_{\mathrm{exp}}$
& \xmark & \xmark & $\alpha_{\mathrm{exp}}=0.28$ & \xmark \\
Interaction radius $R$
& \xmark & \xmark & \xmark & 1.5 \\
Coordinate spacing
& \xmark & \xmark & \xmark & 1.0 \\
Intercomponent gap
& \xmark & \xmark & \xmark & 6.0 \\
\bottomrule
\end{tabular}
}
\end{table*}

\subsubsection{Frontier LLM Analysis}
\label{app:frontier_llm_analysis}

Table~\ref{tab:llm_analysis_parameters} reports the implementation details of the hidden-state analysis for the four frontier LLMs, including GPU allocations, benchmark sampling, dimensionality reduction, and final-layer HDBSCAN clustering. All models follow the same analysis protocol.

\begin{table*}[htbp]
\centering
\small
\caption{Implementation details of the hidden-state analysis across four frontier LLMs.}
\label{tab:llm_analysis_parameters}
\setlength{\tabcolsep}{5pt}

\resizebox{\textwidth}{!}{%
\begin{tabular}{l|cccc}
\toprule
\multirow{2}{*}{\textbf{Hyperparameters}}
& \multicolumn{4}{c}{\textbf{Models}} \\
\cmidrule(l){2-5}
& \textbf{GLM-4.7-Flash}
& \textbf{DeepSeek-V4-Flash}
& \textbf{Kimi-K3}
& \textbf{MiniMax-M3} \\
\midrule

GPUs
& $2\times\text{H100}$
& $4\times\text{H100}$
& $32\times\text{H100}$
& $16\times\text{H100}$ \\

\midrule

Benchmarks
& \multicolumn{4}{c}{
HumanEval, ARC-Easy, HellaSwag, and MATH
} \\

Samples per benchmark
& 100 & 100 & 100 & 100 \\

Hidden-state capture
& Prefill & Prefill & Prefill & Prefill \\

Maximum captured tokens
& 2048 & 2048 & 2048 & 2048 \\

Maximum sequence length
& 2048 & 2048 & 2048 & 2048 \\

Observed Transformer layers
& 12 & 12 & 12 & 12 \\

Hidden-state normalization
& $\ell_2$ & $\ell_2$ & $\ell_2$ & $\ell_2$ \\

\midrule

Reduction methods
& \multicolumn{4}{c}{
UMAP, spherical UMAP
} \\

UMAP clustering dimension
& 32 & 32 & 32 & 32 \\

UMAP neighbors
& 15 & 15 & 15 & 15 \\

UMAP minimum distance
& 0.1 & 0.1 & 0.1 & 0.1 \\

Reducer input metric
& cosine & cosine & cosine & cosine \\

\midrule

Clustering algorithm
& HDBSCAN & HDBSCAN & HDBSCAN & HDBSCAN \\

Clustering distance
& cosine & cosine & cosine & cosine \\

Minimum cluster fraction
& 0.03 & 0.03 & 0.03 & 0.03 \\

Minimum samples fraction
& 0.01 & 0.01 & 0.01 & 0.01 \\

Cluster selection method
& EOM & EOM & EOM & EOM \\

HDBSCAN $\alpha_{\mathrm{HDB}}$
& 1.0 & 1.0 & 1.0 & 1.0 \\

Cluster selection $\epsilon$
& 0.0 & 0.0 & 0.0 & 0.0 \\

Noise-token assignment
& Nearest cluster & Nearest cluster & Nearest cluster & Nearest cluster \\

Cases with undefined $S_{\mathrm{cos}}$
& Report N/A & Report N/A & Report N/A & Report N/A \\

\bottomrule
\end{tabular}
}
\end{table*}

\subsection{Cosine Silhouette Scores}
\label{app:s_cos}

We give the formal definition of the \emph{cosine silhouette score} $S_{\mathrm{cos}}$ used in our analysis. Let $y_i$ denote the reduced representation of token $i$, and let $\mathcal C(i)$ be its assigned cluster. For nonzero vectors, the cosine distance is
\begin{equation}
    d_{\mathrm{cos}}(y_i,y_j)
    =
    1-
    \frac{\langle y_i,y_j\rangle}
    {\lVert y_i\rVert_2\lVert y_j\rVert_2}.
    \label{eq:cosine_distance}
\end{equation}
This distance measures differences in direction rather than magnitude, making it natural for the spherical representations considered here. For unit vectors, it simplifies to $1-\langle y_i,y_j\rangle$.

For a token in a cluster containing at least two members, define
\begin{equation}
    \begin{aligned}
        a_i
        &=
        \frac{1}{|\mathcal C(i)|-1}
        \sum_{\substack{j\in\mathcal C(i)\\j\neq i}}
        d_{\mathrm{cos}}(y_i,y_j),
        \\
        b_i
        &=
        \min_{\mathcal C\neq\mathcal C(i)}
        \frac{1}{|\mathcal C|}
        \sum_{j\in\mathcal C}
        d_{\mathrm{cos}}(y_i,y_j).
    \end{aligned}
    \label{eq:silhouette_cohesion_separation}
\end{equation}
Here, $a_i$ measures the average distance from token $i$ to its own cluster, while $b_i$ measures its average distance to the nearest competing cluster. The tokenwise silhouette and its average over the $N$ evaluated tokens are
\begin{equation}
    s_i
    =
    \frac{b_i-a_i}{\max\{a_i,b_i\}},
    \qquad
    S_{\mathrm{cos}}
    =
    \frac{1}{N}\sum_{i=1}^{N}s_i.
    \label{eq:cosine_silhouette_score}
\end{equation}
The score lies in $[-1,1]$. Values near $1$ indicate that a token is much closer to its own cluster than to any competing cluster; values near $0$ indicate comparable distances to its own and a competing cluster; negative values indicate that another cluster is closer on average. A higher $S_{\mathrm{cos}}$ therefore indicates more compact and better-separated clusters. Because the score compares distances relatively, it measures cluster separability rather than absolute distance to consensus.

The silhouette score is defined only when the clustering contains at least two clusters. Accordingly, we compute $S_{\mathrm{cos}}$ only when $2 \leq K_{\mathrm{cluster}} \leq N-1$, where $K_{\mathrm{cluster}}$ denotes the number of non-noise clusters identified by HDBSCAN. When HDBSCAN identifies fewer than two non-noise clusters, corresponding to either a single non-noise cluster or an all-noise result, we report $S_{\mathrm{cos}}$ as \emph{N/A}.

\subsection{Explicit Opinion Leader Dynamics under Conditional Initialization}
\label{app:explicit_conditional_simulation}

\begin{figure}[htbp]
    \centering
    \includegraphics[width=.99\linewidth]{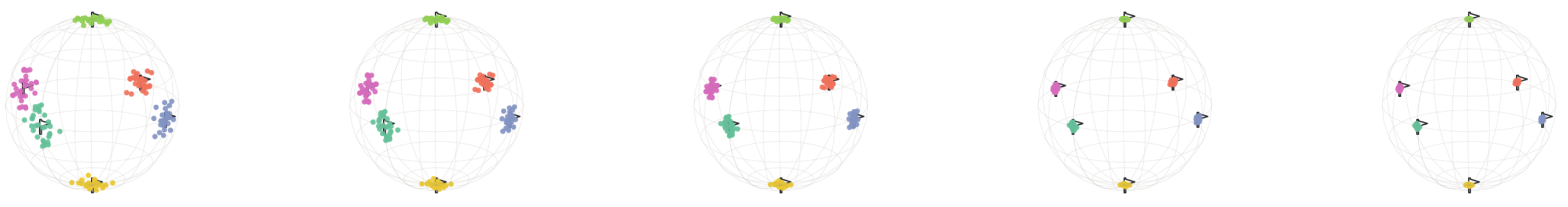}
    \caption{Finite-particle simulation on $\mathbb S^2$ of explicit opinion leader dynamics under conditional initialization. Colors assigned at $t=0$ indicate membership in six groups defined by representatives and are retained throughout the simulation; flags mark the fixed representatives. Unlike general initialization, the conditional setting satisfies stronger concentration and separation conditions, under which the groups remain separated and converge exponentially to distinct \emph{explicit opinion leaders}.}
    \label{fig:conditional_explicit_simulation}
\end{figure}

This experiment complements the general initialization simulation in the main text by considering stronger assumptions. Under conditional initialization, the $n=180$ particles are divided into six equal groups, each initialized with antithetically paired samples drawn from a spherical cap of radius $\delta=0.25$ centered at its associated direction. We set the invariant-cap radius to $\alpha=0.28$ and choose a common representative magnitude such that Assumption~\ref{ass:explicit_separated_basins} holds for every group. Fig.~\ref{fig:conditional_explicit_simulation} shows that the groups remain separated and converge to distinct limiting directions, consistent with the exponential convergence established in Theorem~\ref{thm:explicit_separated_attractors}.

\subsection{Extended Analysis of Token Dynamics in Frontier LLMs}
\label{app:extended_llm_analysis}

\subsubsection{Results across Reduced Embedding Dimensions}
\label{app:embedding_dimensionality}

\begin{table}[htbp]
  \caption{Cosine silhouette scores for \textcolor{explicitcolor}{\textbf{DeepSeek-}}\textcolor{implicitcolor}{\textbf{V4-Flash}} on ARC-Easy across spherical UMAP embedding dimensions. The \emph{upper} row reports mean final-layer scores across $100$ random samples after projecting hidden states with spherical UMAP, while the \emph{lower} row gives the Gaussian reference score at the same hidden dimension.}
  \label{tab:reduced_dimension_effect}
  \centering
  \small
  \setlength{\tabcolsep}{3pt}
  \begin{NiceTabularX}{\linewidth}{
    @{}l *{4}{>{\centering\arraybackslash}X}@{}
  }
    \toprule
    \multirow{2}{*}{\textbf{Model}} & \multicolumn{4}{c}{\textbf{Spherical UMAP embedding dimension}} \\
    \cmidrule(lr){2-5}
    & \textbf{16}
    & \textbf{32}
    & \textbf{48}
    & \textbf{64} \\
    \midrule

    & \textbf{0.86 $\pm$ 0.11} ($\uparrow$ \textbf{0.66})
    & \textbf{0.85 $\pm$ 0.12} ($\uparrow$ \textbf{0.63})
    & \textbf{0.85 $\pm$ 0.11} ($\uparrow$ \textbf{0.64})
    & \textbf{0.85 $\pm$ 0.13} ($\uparrow$ \textbf{0.63}) \\

    \multirow{-2}{*}{\textcolor{explicitcolor}{\textbf{DeepSeek-}}\textcolor{implicitcolor}{\textbf{V4-Flash}}}
    & 0.20 $\pm$ 0.13
    & 0.22 $\pm$ 0.12
    & 0.21 $\pm$ 0.12
    & 0.22 $\pm$ 0.11 \\
    \bottomrule
  \end{NiceTabularX}
\end{table}

We examine the sensitivity of $S_{\mathrm{cos}}$ to the dimension of the reduced embedding. Using \textcolor{explicitcolor}{\textbf{DeepSeek-}}\textcolor{implicitcolor}{\textbf{V4-Flash}} on ARC-Easy, we repeat the quantitative analysis with spherical UMAP embedding dimensions of $16$, $32$, $48$, and $64$. As shown in Table~\ref{tab:reduced_dimension_effect}, the mean cosine silhouette scores remain nearly unchanged at $\mathbf{0.85}$--$\mathbf{0.86}$, while the matched Gaussian references range from $0.20$ to $0.22$. The resulting gaps are similarly stable at $\mathbf{0.63}$--$\mathbf{0.66}$. The consistent results across embedding dimensions support the robustness of the observed clustering behavior across projection settings.

\subsubsection{Results across Dimensionality-Reduction Methods}

\begin{figure}[t]
    \centering
    \includegraphics[width=.99\linewidth]{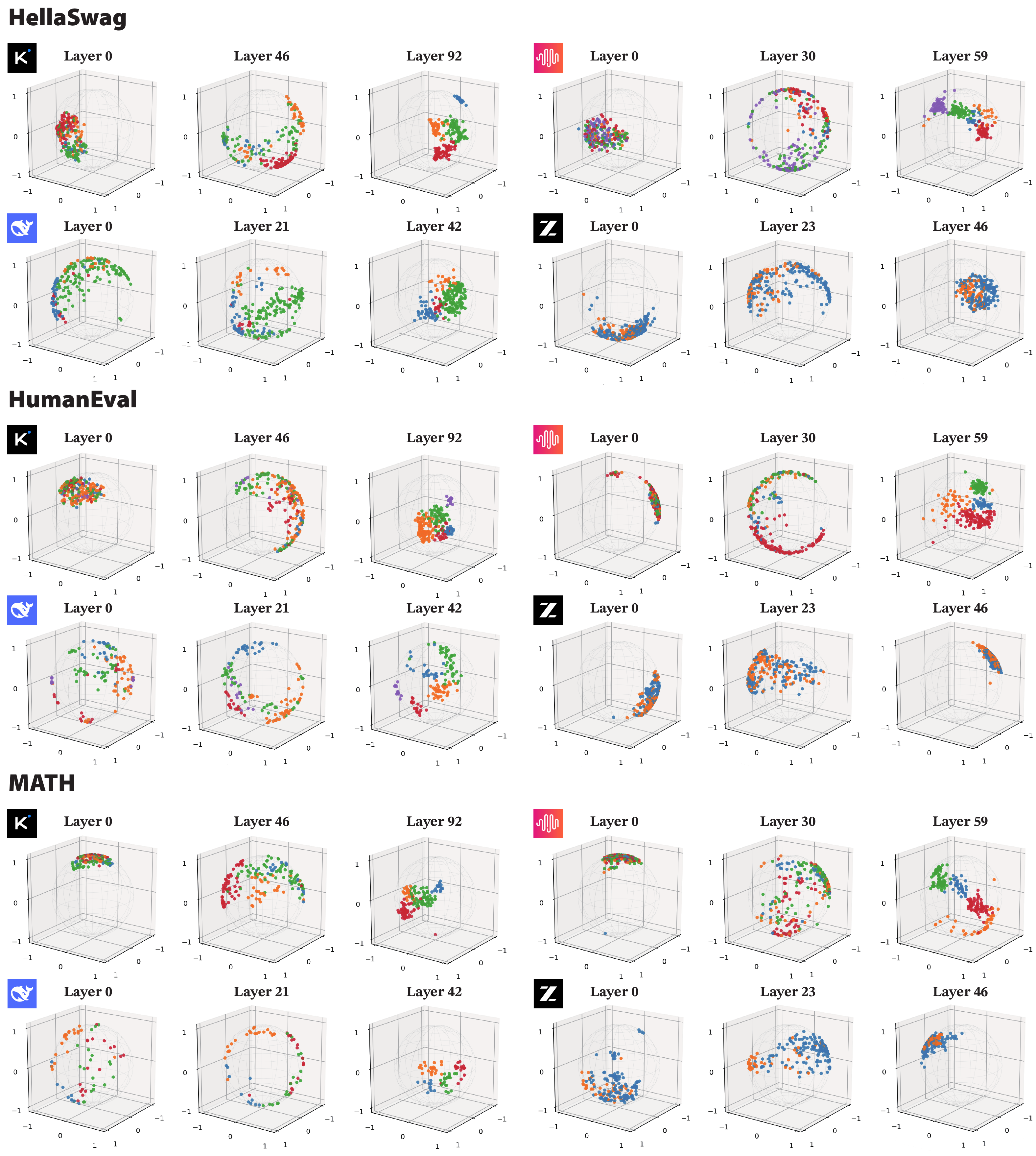}
    \caption{Layerwise evolution of token hidden states for HellaSwag, HumanEval, and MATH samples in \textcolor{explicitcolor}{\textbf{Kimi-K3}}, \textcolor{implicitcolor}{\textbf{MiniMax-M3}}, \textcolor{explicitcolor}{\textbf{DeepSeek-}}\textcolor{implicitcolor}{\textbf{V4-Flash}}, and \textcolor{standardcolor}{\textbf{GLM-4.7-Flash}}, visualized on unit sphere using spherical UMAP.}
    \vspace{-3mm}
    \label{fig:llm_observation_extended_riemann_umap}
\end{figure}

In addition to the spherical UMAP visualizations for ARC-Easy presented in the main text, we extend the same layerwise analysis of token hidden states to HellaSwag, HumanEval, and MATH for all four LLMs. As shown in Fig.~\ref{fig:llm_observation_extended_riemann_umap}, \textcolor{explicitcolor}{\textbf{Kimi-K3}}, \textcolor{implicitcolor}{\textbf{MiniMax-M3}}, and \textcolor{explicitcolor}{\textbf{DeepSeek-}}\textcolor{implicitcolor}{\textbf{V4-Flash}} consistently exhibit distinct final-layer clusters. By contrast, \textcolor{standardcolor}{\textbf{GLM-4.7-Flash}} exhibits weaker group separation and ends with a more concentrated final configuration. These results further demonstrate that the same qualitative pattern holds across benchmarks and aligns with the distinction in our framework between separated token groups and global alignment.

\begin{table}[htbp]
  \caption{
    Cosine silhouette scores across LLMs and benchmarks. For each LLM, the \emph{upper} row reports mean final-layer scores across $100$ random samples after projecting hidden states with UMAP, while the \emph{lower} row gives the Gaussian reference score at the same hidden dimension.}
  \label{tab:extended_umap}
  \centering
  \small
  \setlength{\tabcolsep}{3pt}

  \begin{NiceTabularX}{\linewidth}{
    @{}l *{4}{>{\centering\arraybackslash}X}@{}
  }
    \toprule
    \textbf{Model}
    & \textbf{HumanEval}
    & \textbf{ARC-Easy}
    & \textbf{HellaSwag}
    & \textbf{MATH} \\
    \midrule

    & \textbf{0.76 $\pm$ 0.27} ($\uparrow$ \textbf{0.58})
    & \textbf{0.89 $\pm$ 0.14} ($\uparrow$ \textbf{0.71})
    & \textbf{0.76 $\pm$ 0.18} ($\uparrow$ \textbf{0.61})
    & \textbf{0.90 $\pm$ 0.14} ($\uparrow$ \textbf{0.71}) \\

    \multirow{-2}{*}{\textcolor{explicitcolor}{\textbf{Kimi-K3}}}
    & 0.18 $\pm$ 0.04
    & 0.18 $\pm$ 0.06
    & 0.15 $\pm$ 0.04
    & 0.19 $\pm$ 0.05 \\
    \midrule

    & \textbf{0.87 $\pm$ 0.08} ($\uparrow$ \textbf{0.71})
    & \textbf{0.83 $\pm$ 0.08} ($\uparrow$ \textbf{0.66})
    & \textbf{0.86 $\pm$ 0.08} ($\uparrow$ \textbf{0.70})
    & \textbf{0.84 $\pm$ 0.10} ($\uparrow$ \textbf{0.66}) \\

    \multirow{-2}{*}{\textcolor{implicitcolor}{\textbf{MiniMax-M3}}}
    & 0.16 $\pm$ 0.03
    & 0.17 $\pm$ 0.04
    & 0.16 $\pm$ 0.03
    & 0.18 $\pm$ 0.03 \\
    \midrule

    & \textbf{0.73 $\pm$ 0.16} ($\uparrow$ \textbf{0.54})
    & \textbf{0.85 $\pm$ 0.12} ($\uparrow$ \textbf{0.63})
    & \textbf{0.64 $\pm$ 0.19} ($\uparrow$ \textbf{0.46})
    & \textbf{0.85 $\pm$ 0.11} ($\uparrow$ \textbf{0.64}) \\

    \multirow{-2}{*}{\textcolor{explicitcolor}{\textbf{DeepSeek-}}\textcolor{implicitcolor}{\textbf{V4-Flash}}}
    & 0.19 $\pm$ 0.07
    & 0.22 $\pm$ 0.11
    & 0.18 $\pm$ 0.04
    & 0.21 $\pm$ 0.13 \\
    \midrule

    & 0.34 $\pm$ 0.30 ($\uparrow$ 0.19)
    & 0.22 $\pm$ 0.16 ($\uparrow$ 0.05)
    & 0.32 $\pm$ 0.25 ($\uparrow$ 0.17)
    & 0.30 $\pm$ 0.16 ($\uparrow$ 0.13) \\

    \multirow{-2}{*}{\textcolor{standardcolor}{\textbf{GLM-4.7-Flash}}}
    & 0.15 $\pm$ 0.07
    & 0.17 $\pm$ 0.12
    & 0.15 $\pm$ 0.04
    & 0.17 $\pm$ 0.13 \\
    \bottomrule
  \end{NiceTabularX}
\end{table}

To examine whether such observed patterns depend on the choice of dimensionality-reduction methods, we repeat the analysis with UMAP in place of spherical UMAP. As reported in Table~\ref{tab:extended_umap}, the three sparse-attention LLMs score $\mathbf{0.64}$--$\mathbf{0.90}$ across all benchmarks, exceeding their matched Gaussian references by $\mathbf{0.46}$--$\mathbf{0.71}$. By comparison, \textcolor{standardcolor}{\textbf{GLM-4.7-Flash}} scores $0.22$--$0.34$, with smaller differences of $0.05$--$0.19$. The agreement between the results obtained with UMAP and spherical UMAP indicates that the stronger group separation observed in sparse-attention LLMs is robust across dimensionality-reduction methods, providing further empirical support for the clustering behavior predicted by our theory.

Fig.~\ref{fig:llm_observation_extended_umap} visualizes the layerwise evolution of token hidden states after projection with UMAP. Across all four benchmarks, the sparse-attention LLMs exhibit clearly separated final-layer token groups, whereas \textcolor{standardcolor}{\textbf{GLM-4.7-Flash}} shows less distinct group structure, with its final-layer token representations concentrated in a smaller region. This pattern agrees with the quantitative results in Table~\ref{tab:extended_umap} and closely matches the spherical UMAP findings presented in the main text, indicating that the observed contrast is consistent across dimensionality-reduction methods.

\begin{figure}[htbp]
    \centering
    \includegraphics[width=.99\linewidth]{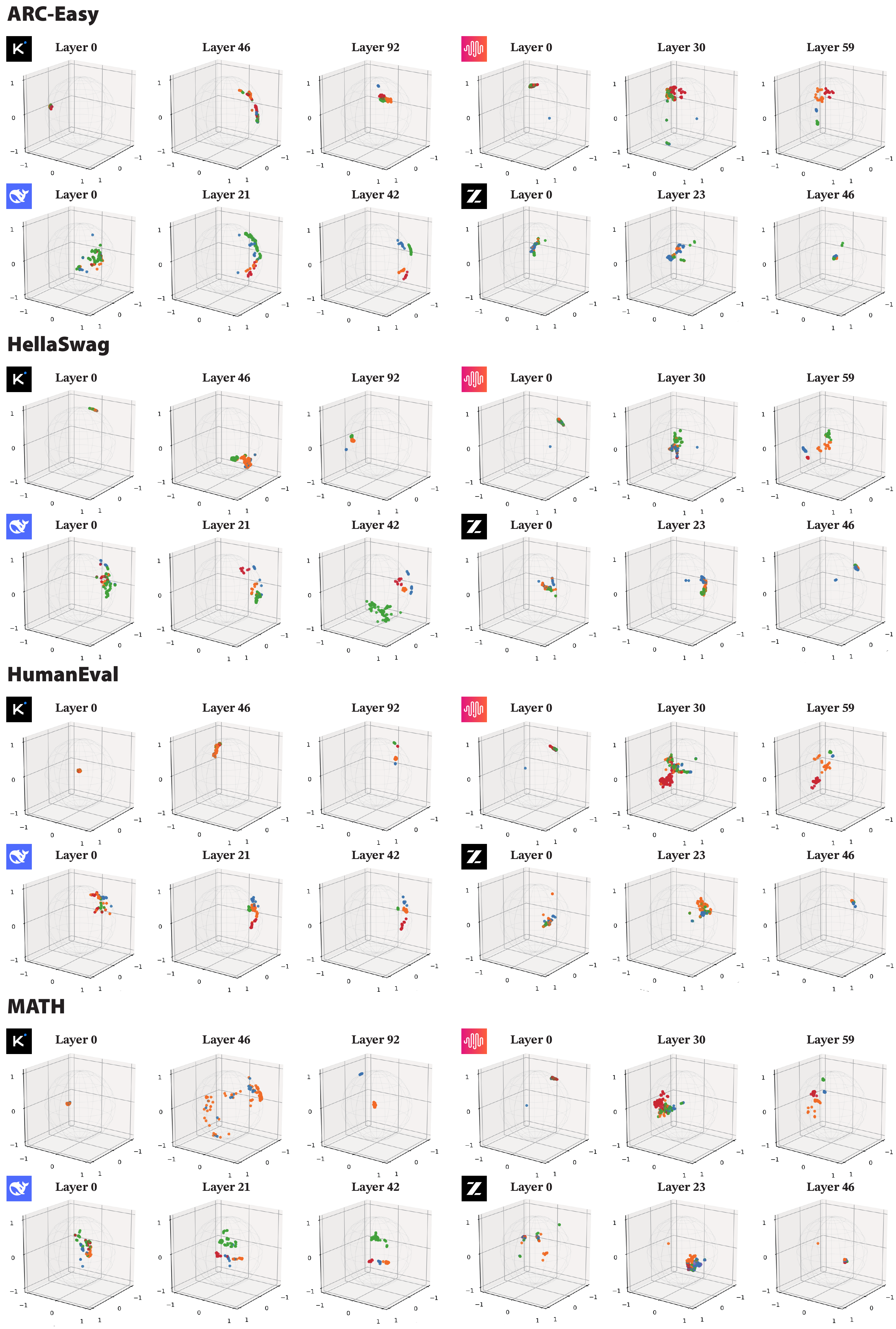}
    \caption{Layerwise evolution of token hidden states for ARC-Easy, HellaSwag, HumanEval, and MATH samples in \textcolor{explicitcolor}{\textbf{Kimi-K3}}, \textcolor{implicitcolor}{\textbf{MiniMax-M3}}, \textcolor{explicitcolor}{\textbf{DeepSeek-}}\textcolor{implicitcolor}{\textbf{V4-Flash}}, and \textcolor{standardcolor}{\textbf{GLM-4.7-Flash}}, visualized on unit sphere using UMAP.}
    \vspace{-3mm}
    \label{fig:llm_observation_extended_umap}
\end{figure}

\subsubsection{Results on RULER Long-Context Variable Tracking}

To further examine token dynamics under longer contexts, we analyze \textbf{4K}-token inputs from the RULER~\citep{hsieh2024ruler} variable-tracking (VT) task with \textcolor{implicitcolor}{\textbf{MiniMax-M3}} and \textcolor{explicitcolor}{\textbf{DeepSeek-}}\textcolor{implicitcolor}{\textbf{V4-Flash}}. Under a fixed selection budget, longer prompts reduce the fraction of available tokens or blocks retained by sparse attention, producing a more restrictive interaction pattern. Following the protocol in our main analysis, we visualize the layerwise evolution of token hidden states with spherical UMAP and report the mean final-layer $S_{\mathrm{cos}}$ over $100$ samples, together with a Gaussian reference matched to each hidden dimension. 

\begin{wraptable}{l}{0.45\textwidth}
  \centering
  \small
  \setlength{\tabcolsep}{3pt}

  \caption{Cosine silhouette scores for \textcolor{explicitcolor}{\textbf{DeepSeek-}}\textcolor{implicitcolor}{\textbf{V4-Flash}} and \textcolor{implicitcolor}{\textbf{MiniMax-M3}} on RULER VT task (\textbf{4K}-token length). The \emph{upper} row reports mean final-layer scores across $100$ random samples after projecting hidden states with spherical UMAP, while the \emph{lower} row gives the Gaussian reference score at the same hidden dimension.}
  \label{tab:4klength}

  \begin{NiceTabularX}{\linewidth}{
    @{}l >{\centering\arraybackslash}X@{}
  }
    \toprule
    \textbf{Model}
    & \textbf{RULER VT} \\
    \midrule

    \multirow{2}{*}{
      \textcolor{explicitcolor}{\textbf{DeepSeek-}}\textcolor{implicitcolor}{\textbf{V4-Flash}}
    }
    & \textbf{ 0.58 $\pm$ 0.19 } \\
    & N/A {\scriptsize($<2$ non-noise clusters)} \\
    \midrule

    \multirow{2}{*}{
      \textcolor{implicitcolor}{\textbf{MiniMax-M3}}
    }
    & \textbf{ 0.82 $\pm$ 0.05 } \\
    & N/A {\scriptsize($<2$ non-noise clusters)} \\
    \bottomrule
  \end{NiceTabularX}
\end{wraptable}

Table~\ref{tab:4klength} reports mean final-layer $S_{\mathrm{cos}}$ scores of $\mathbf{0.82}\pm\mathbf{0.05}$ for \textcolor{implicitcolor}{\textbf{MiniMax-M3}} and $\mathbf{0.58}\pm\mathbf{0.19}$ for \textcolor{explicitcolor}{\textbf{DeepSeek-}}\textcolor{implicitcolor}{\textbf{V4-Flash}}. For both LLMs, HDBSCAN identifies fewer than two non-noise clusters in every matched Gaussian reference sample. Since the cosine silhouette score requires at least two clusters, we report $S_{\mathrm{cos}}$ as \emph{N/A}. The findings provide quantitative evidence of multiple-group organization in the reduced token representation space of both sparse-attention LLMs. The layerwise visualizations in Fig.~\ref{fig:llm_observation_longcontext} further illustrate how this structure develops across layers: token groups that largely overlap at the input layer reorganize and become more distinct and separated towards the final layers. Together, these results suggest that the multiple-group pattern observed in the main analysis persists under longer-context inputs.

\begin{figure}[htbp]
    \centering
    \includegraphics[width=.99\linewidth]{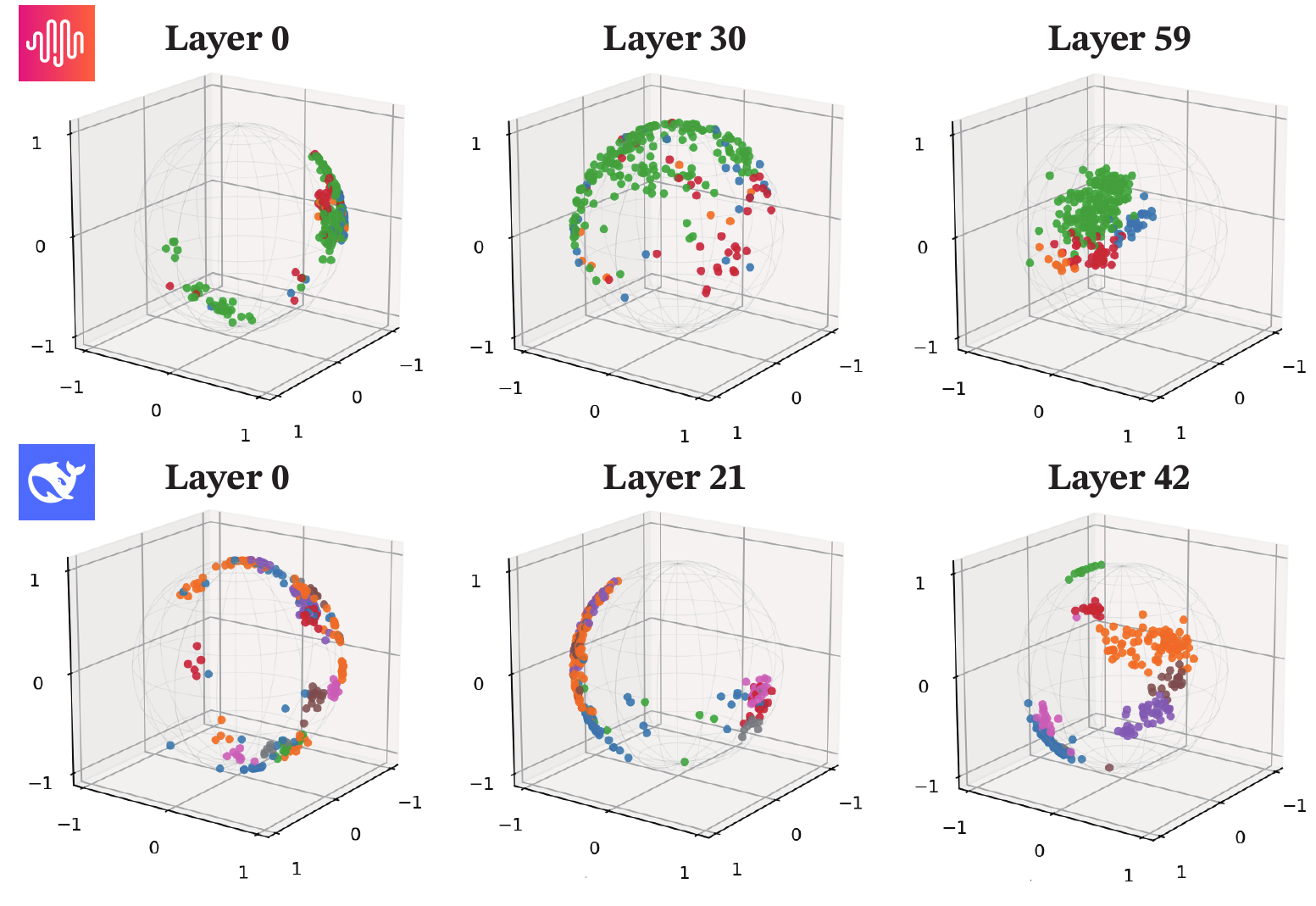}
    \caption{Layerwise evolution of token hidden states for RULER VT samples in \textcolor{implicitcolor}{\textbf{MiniMax-M3}} and \textcolor{explicitcolor}{\textbf{DeepSeek-}}\textcolor{implicitcolor}{\textbf{V4-Flash}}, visualized on unit sphere using spherical UMAP.}
    \vspace{-3mm}
    \label{fig:llm_observation_longcontext}
\end{figure}